\documentclass[11pt, letterpaper]{article}
\pdfoutput=1
\usepackage[authoryear]{natbib}
\usepackage{amsmath,amssymb,amsfonts,amsthm,bbm}
\usepackage{epic,eepic,epsfig,longtable}
\usepackage{multirow,verbatim}
\usepackage{array}

\usepackage{epsfig}
\usepackage{setspace}
\usepackage{color}
\usepackage{pb-diagram}
\usepackage{fancyhdr}
\usepackage{graphicx}
\usepackage{appendix}
\usepackage{listings}
\usepackage{longtable}
\usepackage{url}
\usepackage{lineno,hyperref}
\usepackage{subfigure}
\usepackage{mathrsfs}
\usepackage{stmaryrd}
\usepackage{appendix}
\usepackage{algorithmic}
\usepackage{algorithm}

\makeatletter
\def\singlespace{\def\baselinestretch{1}\@normalsize}

\makeatletter
\def\singlespace{\def\baselinestretch{1}\@normalsize}

\numberwithin{equation}{section}

\renewcommand{\hat}{\widehat}

\renewcommand{\hat}{\widehat}

\newcommand{\bfm}[1]{\ensuremath{\mathbf{#1}}}

   \def\bA{\bfm A}  
   \def\bB{\bfm B}  
   \def\bC{\bfm C}  
   \def\bD{\bfm D}

\def\bg{\bfm g}     
\def\bh{\bfm h}   \def\bH{\bfm H}  
   \def\bI{\bfm I}

   \def\bM{\bfm M}

\def\bp{\bfm p}     \def\PP{\mathbb{P}}
     
   \def\bR{\bfm R}  
\def\bs{\bfm s}   \def\bS{\bfm S}  
     
\def\bu{\bfm u}   \def\bU{\bfm U}  
\def\bv{\bfm v}     
\def\bw{\bfm w}   \def\bW{\bfm W}  
\def\bx{\bfm x}   \def\bX{\bfm X}  
\def\by{\bfm y}   \def\bY{\bfm Y}  
\def\bz{\bfm z}

\newcommand{\bfsym}[1]{\ensuremath{\boldsymbol{#1}}}

           \def\bDelta {\bfsym {\Delta}}

\def\btheta{\bfsym {\theta}}           
          \def\bepsilon{\bfsym \varepsilon}
             \def\bSigma{\bfsym \Sigma}

\def\bxi{\bfsym {\xi}}
\def\bzeta{\bfsym {\zeta}}

\DeclareMathOperator*{\argmin}{argmin}

\DeclareMathOperator{\cov}{cov}

\DeclareMathOperator{\diag}{diag}

\DeclareMathOperator{\E}{E}

\def\newpage{\vfill\eject}

\def\today{\ifcase\month\or
	January\or February\or March\or April\or May\or June\or
	July\or August\or September\or October\or November\or December\fi
	\space\number\day, \number\year}

\newdimen\biblioindent    \biblioindent=30pt

 at 8truept

\newcommand{\beq}{\begin{equation}}
\newcommand{\eeq}{\end{equation}}
\newcommand{\beqn}{\begin{eqnarray}}
\newcommand{\eeqn}{\end{eqnarray}}
\newcommand{\beqnn}{\begin{eqnarray*}}
	\newcommand{\eeqnn}{\end{eqnarray*}}

\def\R{{\mathbb R}}

\def\E{{\mathbb E}}
\def\S{{\mathbb S}}
\def\P{{\mathbb P}}
\def\diag{{\rm diag}}
\def\Tr{{\rm Tr}}

\def\cov{{\rm cov}}

\def\y{{\bf y}}

\def\argmin{{\rm argmin}}

\allowdisplaybreaks
\usepackage{verbatim}
\renewcommand{\baselinestretch}{1.66}
\numberwithin{equation}{section}
\theoremstyle{plain}
\newtheorem{thm}{Theorem}[section]

\newtheorem{defn}{Definition}[section]
\newtheorem{lem}{Lemma}[section]
\newtheorem{cor}{Corollary}[section]

\newtheorem{ass}{Assumption}[section]

\theoremstyle{definition}

\newcounter{CondCounter}

\newcommand{\ltwonorm}[1]{\lVert#1\rVert_2}

\makeatother

\begin{document}

\title{Communication-Efficient Accurate Statistical Estimation}

\author{Jianqing Fan\thanks{Department of ORFE, Princeton University, Princeton, NJ 08544, USA. Email: \texttt{jqfan@princeton.edu}.}
	\and Yongyi Guo\thanks{Department of ORFE, Princeton University, Princeton, NJ 08544, USA. Email: \texttt{yongyig@princeton.edu}.}
	\and Kaizheng Wang\thanks{Department of IEOR, Columbia University, New York, NY 10027, USA. Email: \texttt{kaizheng.wang@columbia.edu}.}
}

\date{August 2021}

\maketitle
\onehalfspacing

\begin{abstract}
When the data are stored in a distributed manner, direct applications of traditional statistical inference procedures are often prohibitive due to communication costs and privacy concerns.
This paper develops and investigates two Communication-Efficient Accurate Statistical Estimators (CEASE), implemented through iterative algorithms for distributed optimization. In each iteration, node machines carry out computation in parallel and communicate with the central processor, which then broadcasts aggregated information to node machines for new updates.
The algorithms adapt to the similarity among loss functions on node machines, and converge rapidly when each node machine has large enough sample size. Moreover, they do not require good initialization and enjoy linear converge guarantees under general conditions. The contraction rate of optimization errors is presented explicitly, with dependence on the local sample size unveiled. In addition, the improved statistical accuracy per iteration is derived.  By regarding the proposed method as a multi-step statistical estimator, we show that statistical efficiency can be achieved in finite steps in typical statistical applications.  In addition, we give the conditions under which the one-step CEASE estimator is statistically efficient.  Extensive numerical experiments on both synthetic and real data validate the theoretical results and demonstrate the superior performance of our algorithms.
\end{abstract}

\section{Introduction}\label{sec:intro}

Statistical inference in modern era faces tremendous challenge on computation and storage. The exceedingly large size of data often makes it impossible to store all of them on a single machine. Moreover, many applications have individual agents (e.g. local governments, research labs, hospitals, smart phones) collecting data independently. Communication is prohibitively expensive due to the limited bandwidth, and direct data sharing raises concerns in privacy and loss of ownership. These constraints make it necessary to develop methodologies for distributed systems, solving statistical problems with divide-and-conquer procedures and communicating only certain summary statistics. 

Distributed statistical inference has received considerable attention recently, covering a wide spectrum of topics including $M$-estimation \citep{ZDW13,CXi14,SSZ14,RNa16,WKS17,LLS17,BFL18,WRX18,SLS18,BDS19}, principal component analysis \citep{FWW17,GSS17}, nonparametric regression \citep{SCh17,Sva17,HMO18}, quantile regression \citep{VCC17,CLZ18}, bootstrap \citep{KTS14}, confidence intervals \citep{JLY18,CLZ18}, Bayesian methods \citep{WDu13,JLY18}, etc. In the commonly-used setting, the overall dataset is partitioned and stored on $m$ node machines connected to a central processor. Most of the approaches only require one round of communication: the node machines work in parallel and send their results to the central processor, which then aggregates the information to get a final result. As typical examples, \cite{ZDW13} average the $M$-estimators on node machines; \cite{BFL18} average debiased estimators; and \cite{FWW17} average subspaces via eigen-decomposition. While these one-shot methods are communication-efficient, they only work with a small number of node machines (e.g. $m=o(\sqrt{N})$, where $N$ is the total sample size) and require large sample on each, as their theories heavily rely on asymptotic expansions of estimators. Such conditions are easily violated in practice.

Multi-round procedures come as a remedy, which alternate between local computations and global aggregations. 
It is possible to achieve optimal statistical precision after a few rounds of communications, under broader settings than those for one-shot procedures. \cite{SSZ14} propose a Distributed Approximate NEwton (DANE) algorithm where, in each iteration, each node machine minimizes a modified loss function based on its own samples and the gradient information from all other machines obtained through communication. However, for non-quadratic losses, the analysis in \cite{SSZ14} does not imply any advantage of DANE in terms of communication over distributed implementation of gradient descent. Other approximate Newton algorithms include \cite{ZXi15}, \cite{WRX18}, \cite{CLZ18} and \cite{CRo19}. \cite{JLY18} develop a Communication-efficient Surrogate Likelihood (CSL) framework for estimation and inference in regular parametric models, penalized regression, and Bayesian statistics. A similar method also appears in \cite{WKS17}. These methods no longer have restrictions on the number of machines such as $m=o(\sqrt{N})$.

Due to the nature of Newton-type methods, existing theories for these algorithms heavily rely on good initialization or even self-concordance assumption on loss functions. They essentially focus on improving an initial estimator that is already consistent but not efficient, whose ideas coincide with the classical one-step estimator \citep{Bic75}. Such initialization itself needs additional efforts and assumptions. Moreover, current results still require each machine to have sufficiently many samples so that loss functions on different machines are similar to each other. These all make the proposed methods unreliable in practice.

Aside from distributed statistical inference, there has also been a vast literature in distributed optimization. The ADMM \citep{BPC11} is a celebrated example among the numerous algorithms that handle deterministic optimization problems with minimum structural assumption. Yet, the convergence can be quite slow and it cannot fully utilize the similarity among loss functions on node machines.

In this paper, we develop and study two Communication-Efficient Accurate Statistical Estimators (CEASE) based on multi-round algorithms for distributed statistical estimation. Our new algorithms extend the DANE algorithm \citep{SSZ14} to regularized empirical risk minimization. Moreover, we provide sharp convergence guarantees for general scenarios, even if the local loss functions are dissimilar and regularization is nonsmooth. 

We assume that all the $m$ node machines have the same sample size $n$. Each has a regularized empirical risk function $f_k+g$ defined by the samples stored there, and the goal is to compute the minimizer of the overall regularized risk function $\frac{1}{m} \sum_{k=1}^{m} f_k+g$ to statistical precision. When $n$ is sufficiently large, their rates of convergence are better than or comparable to existing methods designed for this large-sample regime. For moderate or small $n$, they are still guaranteed to converge linearly even without good initialization, while other statistical methods fail. In addition, our algorithms take advantage of the similarity among $\{ f_k \}_{k=1}^m$ and thus improve over general-purpose algorithms like ADMM. They interpolate between distributed algorithms for statistical estimation and general deterministic problems. Theoretical findings are verified by extensive numerical experiments.
From a technical point of view, our algorithms use the proximal point algorithm \citep{Roc76} as the backbone and obtain inexact updates in a distributed manner. This turns out to be crucial for proving convergence under general conditions. Our techniques are potentially useful for studying other distributed algorithms.

The rest of this paper is organized as follows. Section \ref{sec-setup} introduces the algorithms. Section \ref{sec-large-sample} presents deterministic convergence results. Section \ref{sec-general} provides guarantees in statistical problems. Section \ref{sec-numerical} shows numerical results on both synthetic and real data. Section \ref{sec-discussion} concludes the paper and discusses possible future directions.

Here we list the notations used throughout the paper. We denote by $[n]$ the set $\{ 1,2,\cdots,n \}$. We write $a_n =O( b_n)$ or $a_n\lesssim b_n$ if there exists a constant $C>0$ such that $a_n\leq C b_n$ holds for sufficiently large $n$; and $a_n \asymp b_n$ if $a_n=O(b_n)$ and $b_n = O( a_n )$. Given $\bx ,\by \in \R^k$ and $r>0$, we define $B(\bx,r) = \{ \bz \in \R^k : \| \bz - \bx\|_2 \leq r \}$ and $\langle \bx, \by \rangle = \sum_{j=1}^{k} x_j y_j$. For a convex function $h$ on $\R^k$, we let $\partial h(\bx)$ be its sub-differential set at $\bx \in \R^k$, and $\argmin_{\bx \in \R^k} h(\bx)$ be the set of its minimizers if $\inf_{\bx \in \R^k} h(\bx) > -\infty$. We use $\| \cdot \|_2$ to denote the $\ell_2$ norm of a vector or operator norm of a matrix. For two sequences of random variables $\{ X_n \}_{n=1}^{\infty}$ and $\{ Y_n \}_{n=1}^{\infty}$ where $Y_n \geq 0$, we write $X_n=O_{\P}(Y_n)$ if for any $\varepsilon>0$ there exists $C>0$ such that $\P (|X_n|\geq CY_n) \leq \varepsilon$ for sufficiently large $n$.
We use $\| X \|_{\psi_2} = \sup_{p\geq 1}\frac{1}{\sqrt{p}} \E^{1/p} |X|^p$ to refer to the sub-Gaussian norm of random variable $X$, and $\| \bX \|_{\psi_2} = \sup_{\| u \|_2=1} \| \langle u , \bX \rangle \|_{\psi_2} $ to denote the sub-Gaussian norm of random vector $\bX$.

\section{The CEASE algorithm}\label{sec-setup}

\subsection{Problem setup}

Let $\mathcal{P}$ be an unknown probability distribution over some sample space $\mathcal X$. For any parameter $\btheta \in \R^p$, define its population risk $F(\btheta) = \E_{\bX \sim \mathcal{P} } \ell(\btheta;\bX)$ based on a loss function $\ell: \R^p\times \mathcal{X} \to \R$. In parametric inference problems, $\ell$ is often chosen as the negative log-likelihood function of some parametric family. Under mild conditions, $F$ is well-defined and has a unique minimizer $\btheta^*$. A ubiquitous problem in statistics and machine learning is to estimate $\btheta^*$ given i.i.d. samples $\{ \bX_i \}_{i=1}^N$ from $\mathcal{P}$, and the minimizer of the empirical risk
$f(\btheta)=\frac{1}{N} \sum_{i=1}^{N} \ell (\btheta; \bX_i )$ becomes a natural candidate. To achieve desirable precision in high-dimensional problems, it is often necessary to incorporate prior knowledge of $\btheta^*$. A principled approach is the regularized empirical risk minimization
\begin{align}
	\min_{\btheta \in \R^p } \left\{  f (\btheta ) + g(\btheta) \right\},
	\label{eqn-rerm}
\end{align}
where $g(\btheta)$ is a deterministic penalty function. Common choices for $g(\btheta)$ include the $\ell_2$ penalty $\lambda \|\btheta\|_2^2$ \citep{HKe70}, the $\ell_1$ penalty $\lambda \| \btheta\|_1$ \citep{Tib96}, and a family of folded concave penalty functions $\|p_\lambda(|\btheta|)\|_1$ such as SCAD \citep{fan2001variable} and MCP \citep{Zha10}, where $\lambda >0$ is a regularization parameter. Throughout the paper, we assume that both $\ell$ and $g$ are convex in $\btheta$, and $\ell$ is twice continuously differentiable in $\btheta$. We allow $g$ to be non-smooth (e.g. the $\ell_1$ penalty).

Consider the distributed setting where the $N$ samples are stored on $m$ machines connected to a central processor. Denote by $\mathcal{I}_k$ the index set of samples on the $k$th machine and $f_k(\btheta) = \frac{1}{|\mathcal{I}_k|} \sum_{i \in \mathcal{I}_k } \ell ( \btheta;\bX_i)$. For simplicity, we assume that $\{ \mathcal{I}_k \}_{k=1}^m$ are disjoint, $N$ is a multiple of $m$, and $|\mathcal{I}_k|=n=N/m$ for all $k\in[m]$. Then (\ref{eqn-rerm}) can be rewritten as
\begin{align}
	\min_{\btheta \in \R^p } \left\{ f(\btheta) + g(\btheta) \right\}, \qquad
	f(\btheta) = \frac{1}{m} \sum_{k=1}^{m} f_k (\btheta).
	\label{eqn-distributed}
\end{align}
Each machine $k$ only has access to its local data and hence local loss function $f_k$ and the penalty $g$. We aim to solve (\ref{eqn-distributed}) in a distributed manner with both statistical efficiency and communication-efficiency.

\subsection{Adaptive gradient enhancements and distributed algorithms in large-sample regimes}\label{section2-2}

In the large-sample regime, we drop the regularization term for now and consider the empirical risk minimization problem $\min_{\btheta \in \R^p } f(\btheta)$ for estimating $\btheta^* = \argmin_{\btheta \in \R^p} F(\btheta)$. In some problems, direct minimization of $f$ is costly, while it is easy to obtain some rough estimate $\bar\btheta$ that is close to $\btheta^*$ but not as accurate as the global minimimizer $\widehat\btheta = \argmin_{\btheta \in \R^p } f(\btheta)$. \cite{Bic75} proposes the one-step estimator based on the local quadratic approximation
and shows that it is as efficient as  $\widehat\btheta$ if the initial estimator $\bar\btheta$ is {accurate} enough. Iterating this further results in multiple-step estimators that improve the optimization error and hence statistical errors when the initial estimator is not good enough \citep{robinson1988stochastic}. This inspires us to refine an existing estimator using some proxy of $f$.

In the distributed environment, starting from an initial estimator $\bar \btheta$, the gradient vector $\nabla f(\bar\btheta)$ can easily be communicated. Construct a linear function $f^{(1)} (\btheta) = f(\bar\btheta)+\langle \nabla f(\bar\btheta) , \btheta - \bar\btheta \rangle$, the first-order Taylor expansion of $f$ around $\bar\btheta$.  The object function to be minimized can be written as
$$
f(\btheta) = f^{(1)} (\btheta) + R(\btheta), \qquad \mbox{where} \qquad R (\btheta) = f (\btheta) - f^{(1)} (\btheta).
$$
Since the linear function $f^{(1)} (\btheta)$ can easily be communicated to each node machine whereas $R(\cdot)$ can not, the latter is naturally replaced by its subsampled version at node $k$:
$$
R_k (\btheta) = f_k(\btheta) - [ f_k(\bar\btheta)+\langle \nabla f_k(\bar\btheta) , \btheta - \bar\btheta \rangle ],
$$
where $f_k(\btheta)$ is the loss function based on the data at node $k$.  With this replacement, the target of optimization at node $k$ becomes $f^{(1)}(\btheta) + R_k(\btheta)$, which equals to
$$
f_{k}(\btheta) - \langle \nabla f_{k}(\bar\btheta) - \nabla f(\bar\btheta) ,\btheta\rangle
$$
up to an additive constant.  This function will be called {gradient-enhanced loss function}, in which the gradient at point $\bar\btheta$ based on the local data is replaced by the global one.   This function has one very nice fixed point at the global minimum $\widehat \btheta$: the minimizer of the adaptive gradient-enhanced function at $\bar \btheta = \widehat \btheta$ is still $\widehat{\btheta}$.  This can easily be seen by verifying that the gradient at the point $\widehat{\btheta}$ is zero.

The idea of using such an adaptive gradient-enhanced function has been proposed in \cite{SSZ14} and \cite{JLY18}, though the motivations are different.  \cite{JLY18} develop a Commmunication-efficient Surrogate Likelihood (CSL) method using the {gradient-enhanced loss function}  $ f_1(\btheta) - \langle \nabla f_1(\bar\btheta) - \nabla f(\bar\btheta) ,\btheta\rangle$ on the first machine, uses the minimizer on that machine as a new estimate, and iterates these steps until convergence. In the presence of a regularizer $g$ in (\ref{eqn-rerm}), one simply adds $g$ to the gradient-enhanced loss; see the Algorithm \ref{alg-CSL} below.

\begin{algorithm}[t]
	\caption{CSL \citep{JLY18} }	
	\label{alg-CSL}\begin{algorithmic}		
		\STATE \textbf{{Input}}:
		Initial value $\btheta_{0}$, number of iterations $T$.
		\STATE \textbf{For} $t=0,1,2,\cdots,T-1$:
		\begin{itemize}
			\item Each machine evaluates $\nabla f_{k}(\btheta_{t})$ and sends to the
			$1$st machine;
			\item The $1$st machine computes $\nabla f(\btheta_{t})=\frac{1}{m} \sum_{k=1}^{m} \nabla f_{k}(\btheta_{t})$ and
			\begin{align*}
				\btheta_{t+1}=\argmin_{\btheta}\left\{ f_{1}(\btheta) + g(\btheta) - \langle \nabla f_{1}(\btheta_{t}) - \nabla f(\btheta_{t}) ,\btheta\rangle \right\}
			\end{align*}
			and broadcasts to other machines.
		\end{itemize}
		\STATE \textbf{{Output}}: $\btheta_{T}$.
	\end{algorithmic}
\end{algorithm}

Note that in Algorithm \ref{alg-CSL}, only the first machine solves optimization problems and others just evaluate gradients.  These machines are idling while the first one is working hard. To fully utilize the computing power of machines and accelerate convergence, all the machines can optimize their corresponding {gradient-enhanced loss functions} in parallel and the central processor then aggregates the results.
This is motivated by the Distributed Approximate NEwton (DANE) algorithm \citep{SSZ14}. Algorithm \ref{alg-vanilla} describes the procedure in detail. Intuitively, the averaging step requires little computation but helps reduce the variance of estimators on node machines and enhance the accuracy.

\begin{algorithm}[t]
	\caption{{Distributed estimation using gradient-enhanced loss}}	
	\label{alg-vanilla}\begin{algorithmic}		
		\STATE \textbf{{Input}}:
		Initial value $\btheta_{0}$, number of iterations $T$.
		\STATE \textbf{For} $t=0,1,2,\cdots,T-1$:
		\begin{itemize}
			\item Each machine evaluates $\nabla f_{k}(\btheta_{t})$ and sends to the
			central processor;
			\item The central processor computes $\nabla f(\btheta_{t})=\frac{1}{m} \sum_{k=1}^{m} \nabla f_{k}(\btheta_{t})$
			and broadcasts to machines;
			\item Each machine computes
			\begin{align*}
				\btheta_{t,k}=\argmin_{\btheta}\left\{ f_{k}(\btheta) + g(\btheta) - \langle \nabla f_{k}(\btheta_{t}) - \nabla f(\btheta_{t}) ,\btheta\rangle \right\}
			\end{align*}
			and sends to the central processor;
			\item The central processor computes $\btheta_{t+1}=\frac{1}{m}\sum_{k=1}^{m}\btheta_{t,k}$
			and broadcasts to machines.
		\end{itemize}
		\STATE \textbf{{Output}}: $\btheta_{T}$.
	\end{algorithmic}
\end{algorithm}

We now illustrate Algorithm \ref{alg-vanilla} in the context of linear regression. Given samples $\{(\bx_i, y_i)\}_{i\in[N]}$, the $k$th machine defines a quadratic loss function
\[
f_k(\btheta) = \frac{1}{2n} \sum_{i\in\mathcal I_k} (y_{i} - \bx_{i} ^{\top} \btheta )^2 = \frac{1}{2}\btheta^{\top}\widehat{\bSigma}_{k}\btheta-\widehat{\bw}_{k}^{\top}\btheta + \frac{1}{2n} \sum_{i \in \mathcal{I}_k} y_i^2.
\]
Here, $\widehat\bSigma_k = \frac{1}{n} \sum_{i\in\mathcal I_k} \bx_{i} \bx_{i}^{\top}$ and $\widehat\bw_k = \frac{1}{n} \sum_{i\in\mathcal I_k} \bx_{i} y_{i}$. The overall loss function is $f(\btheta)= \frac{1}{m} \sum_{k=1}^m f_k(\btheta)
=\frac{1}{2}\btheta^{\top}\widehat{\bSigma}\btheta-\widehat{\bw}^{\top}\btheta$, where $\widehat \bSigma = \frac{1}{m}\sum_{k=1}^m \widehat\bSigma_k$, $\widehat \bw = \frac{1}{m}\sum_{k=1}^m \widehat\bw_k$. Then the update of Algorithm \ref{alg-vanilla} in one iteration is
\begin{align}
\btheta_{t+1,k} & =(\bI-\widehat{\bSigma}_{k}^{-1}\widehat{\bSigma})\btheta_{t}+\widehat{\bSigma}_{k}^{-1}\widehat{\bw},
\label{eqn-dane1}
\\
\btheta_{t+1} & =\bigg( \bI- \frac{1}{m}\sum_{k=1}^{m}\widehat{\bSigma}_{k}^{-1} \widehat{\bSigma}  \bigg) \btheta_{t}+ \frac{1}{m}\sum_{k=1}^{m}\widehat{\bSigma}_{k}^{-1} \widehat{\bw}.
\label{eqn-dane2}
\end{align}
Intuitively, this is a form of contraction towards the global minimizer $\widehat\btheta$. As for the logistic regression, we can also write out the corresponding enhanced losses and minimize them using Newton's method. Due to space limitations, we refer to Appendix C for details.

\subsection{The CEASE Algorithm in general regimes}\label{section2-3}

Algorithms \ref{alg-CSL} and \ref{alg-vanilla} are built upon large-sample regimes, with sufficiently strong convexity of $\{ f_k+g \}_{k=1}^m$ and small discrepancy between them. This requires the local sample size $n$ to be large enough, which may not be the case in practice. Even worse, the required local sample size depends on structural parameters, making such a condition unverifiable.  In fact, our numerical experiments confirm the instability of Algorithms \ref{alg-CSL} and \ref{alg-vanilla} even for moderate $n$. A naive method of remedy is to add strict convex quadratic regularization $q(\btheta)$. While this remedy can make the algorithm converge rapidly, the nonadaptive nature of $q(\btheta)$ will lead to a wrong target. Instead of using a fixed $q$, we will adjust it according to current solutions. The idea stems from the proximal point algorithm \citep{Roc76}.

\begin{defn}
	For any convex function $h:~\R^p \to \R$, define the proximal mapping $\mathrm{prox}_{h}:~\R^p\to\R^p$, $\bx\mapsto \argmin_{\by \in \R^p } \{ h(\by) + \| \by-\bx\|_2^2/2 \}$.
\end{defn}
For a given $\alpha > 0 $, the proximal point algorithm for minimizing $h$ iteratively computes
\begin{align*}
	\bx_{t+1} = \mathrm{prox}_{\alpha^{-1}h } ( \bx_t ) = \argmin_{\bx \in \R^p} \{ h(\bx)+(\alpha/2)\|\bx-\bx_t\|_2^2 \},\qquad \forall t\geq 0,
\end{align*}
starting from some initial value $\bx_0$. It is a strongly convex optimization, shrinking towards the current value $\bx_t$. Under mild conditions, $\{ \bx_t \}_{t=0}^{\infty}$ converges linearly to some $\widehat\bx \in \argmin_{\R^p} h(\bx)$ \citep{Roc76}.

Now we take $h=f+g$ and write the proximal point iteration for our problem (\ref{eqn-distributed}):
\begin{align}
	\btheta_{t+1} = \mathrm{prox}_{\alpha^{-1}(f+g) } ( \btheta_t ) = \argmin_{\btheta \in \R^p} \left\{ f(\btheta)+g(\btheta)+ \frac{\alpha}{2} \| \btheta - \btheta_t\|_2^2 \right\}.
	\label{eqn-prox-alg}
\end{align}
Each iteration (\ref{eqn-prox-alg}) is a distributed optimization problem, whose object function is not available to node machines.  But it can be solved by Algorithms \ref{alg-CSL} and \ref{alg-vanilla}. Specifically, suppose we have already obtained $\btheta_t$ and aim for $\btheta_{t+1}$ in (\ref{eqn-prox-alg}). Letting $\tilde g (\btheta) = g(\btheta) + (\alpha/2) \|\btheta - \btheta_t\|_2^2$, Algorithm \ref{alg-vanilla} starting from $\tilde{\btheta}_0=\btheta_t$ produces iterations over $s = 0, 1, \cdots$
\begin{align*}
	\tilde\btheta_{s,k}
	&= \argmin_{\btheta \in \R^p} \left\{  f_k(\btheta) + \tilde g(\btheta) + \langle \nabla f_{k}(\tilde\btheta_{s}) - \nabla f(\tilde\btheta_{s}) ,\btheta\rangle \right\},\qquad k\in[m],\\
	\tilde{\btheta}_{s+1}&= \frac{1}{m} \sum_{k=1}^{m} \tilde\btheta_{s,k}.
\end{align*}
When $\alpha+\rho_0>\delta$, $\{ \tilde{\btheta}_s \}_{s=0}^{\infty}$ converges $Q$-linearly \footnote{According to \cite{NWr06}, a sequence $\{ \bx_n \}_{n=1}^\infty$ in $\R^p$ is said to converge $Q$-linearly to $\bx^* \in \R^p$ if there exists $r\in(0,1)$ such that $\|\bx_{n+1} - \bx^*\|_2\leq r \| \bx_n - \bx^*\|_2$ for $n$ sufficiently large.} to $\btheta_{t+1}$. On the other hand, there is no need to solve (\ref{eqn-prox-alg}) exactly, as $\mathrm{prox}_{\alpha^{-1}(f+g) } ( \btheta_t )$ is merely an intermediate quantity for computing $\widehat\btheta$. We therefore only run one iteration of Algorithm \ref{alg-vanilla} and use the resulting approximate solution as $\btheta_{t+1}$. This considerably simplifies the algorithm, reducing double loops to a single loop, and enhances statistical interpretation of the method as a multi-step estimator.  However, it makes technical analysis more challenging.  Similarly, we may also use one step of  Algorithm \ref{alg-CSL} to compute the inexact proximal update.

The above discussions lead us to propose two Communication-Efficient Accurate Statistical Estimators (CEASE) in Algorithms \ref{alg-new-single} and \ref{alg-new}, which use the proximal point algorithm as the backbone and obtain inexact updates in a distributed manner. They are regularized versions of Algorithms \ref{alg-CSL} and \ref{alg-vanilla}, with an additional proximal term in the objective functions. That term reduces relative differences of the local loss functions on individual machines, and is crucial for convergence when $\{ f_k \}_{k=1}^m$ are not similar enough. In Appendix A we introduce a variant of Algorithm \ref{alg-new} which also stablizes Algorithm \ref{alg-vanilla}.

\begin{algorithm}[t]
	\caption{ Communication-Efficient Accurate Statistical Estimators (CEASE) }	
	\label{alg-new-single}\begin{algorithmic}		
		\STATE \textbf{{Input}}:
		Initial value $\btheta_{0}$, regularizer $\alpha \geq 0$, number of iterations $T$.
		
		\STATE \textbf{For} $t=0,1,2,\cdots,T-1$:
		\begin{itemize}
			\item Each machine evaluates $\nabla f_{k}(\btheta_{t})$ and sends to the
			$1^{st}$ machine;
			\item The $1^{st}$ machine computes $\nabla f(\btheta_{t})=\frac{1}{m} \sum_{k=1}^{m} \nabla f_{k}(\btheta_{t})$ and
			\[
			\btheta_{t+1}=\argmin_{\btheta}\left\{ f_{1}(\btheta) + g(\btheta) - \langle \nabla f_{k}(\btheta_{t}) - \nabla f(\btheta_{t}) ,\btheta\rangle+\frac{\alpha}{2}\ltwonorm{\btheta-\btheta_{t}}^{2}\right\},
			\]
			and broadcasts to other machines.
		\end{itemize}
		\STATE \textbf{{Output}}: $\btheta_{T}$.	
	\end{algorithmic}
\end{algorithm}

\begin{algorithm}[t]
	\caption{CEASE with averaging}	
	\label{alg-new}\begin{algorithmic}		
		\STATE \textbf{{Input}}:
		Initial value $\btheta_{0}$, regularizer $\alpha \geq 0$, number of iterations $T$.
		
		\STATE \textbf{For} $t=0,1,2,\cdots,T-1$:
		\begin{itemize}
			\item Each machine evaluates $\nabla f_{k}(\btheta_{t})$ and sends to the
			central processor;
			\item The central processor computes $\nabla f(\btheta_{t})=\frac{1}{m} \sum_{k=1}^{m} \nabla f_{k}(\btheta_{t})$
			and broadcasts to machines;
			\item Each machine computes			
			\[
			\btheta_{t,k}=\argmin_{\btheta}\left\{ f_{k}(\btheta)+ g(\btheta) - \langle \nabla f_{k}(\btheta_{t}) - \nabla f(\btheta_{t}) ,\btheta\rangle+\frac{\alpha}{2}\ltwonorm{\btheta-\btheta_{t}}^{2}\right\}
			\]
			and sends to the central processor;
			\item The central processor computes $\btheta_{t+1}=\frac{1}{m}\sum_{k=1}^{m}\btheta_{t,k}$
			and broadcasts to machines.
		\end{itemize}
		\STATE \textbf{{Output}}: $\btheta_{T}$.	
	\end{algorithmic}
\end{algorithm}

In each iteration, Algorithm \ref{alg-new-single} has one round of communication and one optimization problem to solve. Although Algorithm \ref{alg-new} has two rounds of communication per iteration, only one round involves parallel optimization and the other is simply averaging. We will compare their theoretical guarantees as well as practical performances in the sequel. 

Algorithm \ref{alg-new} is an extension of the DANE algorithm in \cite{SSZ14} to regularized empirical risk minimization. While DANE is originally motivated by mirror descent, we view it as a \emph{distributed implementation of the proximal point algorithm}. The new perspective helps us obtain stronger convergence guarantees.
Ideas from the proximal point algorithm have appeared in the literature of distributed stochastic optimization for different purposes such as accelerating first-order algorithms \citep{LLM17} and regularizing sizes of updates \citep{WWS17}.

\section{Deterministic analysis}\label{sec-large-sample}

We first present in Section \ref{section-new-deterministic} the deterministic (almost sure) results for Algorithms \ref{alg-new-single} and \ref{alg-new} based on high-level structural assumptions. As special cases of these algorithms with $\alpha = 0$, Algorithms \ref{alg-CSL} and \ref{alg-vanilla} will be analyzed in Section \ref{sec-vanilla-deterministic}.

\subsection{Deterministic analysis of the CEASE algorithm}\label{section-new-deterministic}

\begin{defn}
	Let $h:\R^p \to \R$ be a convex function, $\Omega \subseteq \R^p$ be a convex set, and $\rho\geq 0$. $h$ is $\rho$-strongly convex in $\Omega$ if $h(\by) \geq h(\bx) + \langle \bg , \by - \bx \rangle + \frac{\rho}{2} \| \by - \bx \|_2^2 $, $\forall \bx, \by \in \Omega$ and $ \bg \in \partial h(\bx) $.
\end{defn}

\begin{ass}[Strong convexity]\label{assump-cvxity}
	$f+g$ has a unique minimizer $\widehat\btheta \in \R^p$, and is $\rho$-strongly convex in $B(\widehat\btheta,R)$ for some $R>0$ and $\rho >0$.
\end{ass}
\begin{ass}[Homogeneity]\label{assump-similarity}
	$\| \nabla^2 f_k(\btheta) - \nabla^2 f(\btheta) \|_2 \leq \delta$, $\forall k\in[m], \btheta \in B(\widehat\btheta,R)$.
\end{ass}

We will refer to $\delta$ as a homogeneity parameter.  Based on both assumptions, we define
\begin{align}
	\rho_0 = \sup \left\{ c \in [0, \rho] : \{ f_k + g \}_{k=1}^m \text{ are } c\text{-strongly convex in } B(\widehat\btheta,R) \right\}.
	\label{eqn-rho0}
\end{align}
A simple but useful fact is $\max\{ \rho - \delta , 0 \} \leq \rho_0 \leq \rho$.
In most interesting problems, the population risk $F$ is smooth and strongly convex on any compact set. When $\{ \bX_i \}_{i=1}^N$ are i.i.d. and the total sample size $N$ is large, the empirical risk $f$ concentrates around $F$ and inherits nice properties from the latter, making Assumption \ref{assump-cvxity} hold easily.

On the other hand, since the empirical risk functions $\{ f_k \}_{k=1}^m$ are i.i.d. stochastic approximations of the population risk $F$, they should not be too far away from their average $f$ provided that $n$ is not too small. Assumption \ref{assump-similarity} is a natural way of characterizing this similarity. It is a generalization of the concept ``$\delta$-related functions" for quadratic losses in \cite{ASh15}. With high probability, it holds with reasonably small $\delta$ and large $R$ under general conditions. Large $n$ implies small homogeneity parameter $\delta$ and thus similar $\{ f_k \}_{k=1}^m$.
Assumption \ref{assump-similarity} always holds with $\delta = \max_{k\in[m]} \{\sup_{\btheta \in B(\widehat\btheta,R)}\| \nabla^2 f_k(\btheta) - \E \nabla^2 f_k(\btheta) \|_2 + \sup_{\btheta \in B(\widehat\btheta,R)} \| \nabla^2 f(\btheta) - \E \nabla^2 f_k(\btheta) \|_2\}$.

The following additional assumption on smoothness of the Hessian matrix of $f+g$ is not necessary for contraction, but it helps us obtain a much stronger result on the contraction rate of Algorithm \ref{alg-vanilla}, justifying the power of the simple averaging step.

\begin{ass}[Smoothness of Hessian]\label{assump-hessian-lip}
	$g\in C^2(\R^p)$, and there exists $M \geq 0$ such that $\| [ \nabla^2 f(\btheta') + \nabla^2 g(\btheta') ] - [ \nabla^2 f(\btheta'') + \nabla^2 g(\btheta'') ] \|_2 \leq M \| \btheta'-\btheta''\|_2$, $\forall \btheta',\btheta'' \in B(\widehat\btheta,R)$.
\end{ass}

Theorem \ref{thm-main} gives contraction guarantees for Algorithms \ref{alg-new-single} and \ref{alg-new}. It is deterministic and non-asymptotic by nature. 

\begin{thm}\label{thm-main}
	Let Assumptions \ref{assump-cvxity} and \ref{assump-similarity} hold. Consider the  multi-step estimators $\{ \btheta_t \}_{t=0}^{T}$ generated by Algorithm \ref{alg-new-single} or \ref{alg-new}. Suppose that $\btheta_0 \in B( \widehat\btheta , R/2 )$ and $[ \delta / ( \rho_0+\alpha ) ]^{2} < \rho / ( \rho+2\alpha )$.
	\begin{enumerate}
		\item [(a)]For both Algorithms \ref{alg-new-single} and \ref{alg-new}, we have
		\begin{align}
			\| \btheta_{t+1} - \widehat\btheta \|_2 \leq \| \btheta_{t} - \widehat\btheta\|_2 \cdot  \frac{ \frac{\delta}{\rho_0+\alpha} \sqrt{\rho^2 + 2\alpha\rho} + \alpha }{\rho+\alpha}, \qquad 0\leq t \leq T-1;
			\label{ineq-thm-main-1}
		\end{align}
		\item  [(b)] If Assumption \ref{assump-hessian-lip} also holds, then for Algorithm \ref{alg-new} we have
		\begin{align}
			\| \btheta_{t+1} - \widehat\btheta \|_2 \leq \| \btheta_t - \widehat\btheta\|_2 \cdot  \frac{ \gamma_t \sqrt{\rho^2 + 2\alpha\rho} + \alpha }{\rho+\alpha},\qquad 0\leq t\leq T-1,
			\label{ineq-thm-main-2}
		\end{align}
		where we define $\gamma_t =
		\frac{\delta}{\rho_0+\alpha} \cdot \min \{ 1 , \frac{\delta}{\rho+\alpha} (1+ \frac{M}{\rho_0+\alpha} \| \btheta_t - \widehat\btheta\|_2 ) \}$;
		\item  [(c)] Both multiplicative factors in (\ref{ineq-thm-main-1}) and (\ref{ineq-thm-main-2}) are strictly less than 1.
	\end{enumerate}
\end{thm}

In the contraction factor in (\ref{ineq-thm-main-1}), the two summands $ \frac{\delta \sqrt{\rho^2 + 2\alpha\rho} }{(\rho_0+\alpha)^2} $ and $\frac{\alpha }{\rho+\alpha}$ come from bounding the inexact proximal update $\| \btheta_{t+1} - \mathrm{prox}_{\alpha^{-1} (f+g) } (\btheta_t) \|_2$ and the residual $\| \mathrm{prox}_{\alpha^{-1} (f+g) } (\btheta_t) - \widehat\btheta \|_2$, respectively. Similar results hold for (\ref{ineq-thm-main-2}). The condition $[ \delta / ( \rho_0+\alpha ) ]^{2} < \rho / ( \rho+2\alpha )$ ensures that both contraction factors are less than 1. Note that \eqref{ineq-thm-main-2} requires Assumption \ref{assump-hessian-lip}, which forces $g$ to be smooth.

Theorem \ref{thm-main} shows the $Q$-linear convergence of the sequence $\{ \btheta_t \}_{t=0}^{\infty}$ generated by both Algorithms \ref{alg-new-single} and \ref{alg-new} under quite general settings. The contraction rate depends explicitly on the structural parameters and the choice of $\alpha$. The local loss functions $\{ f_k \}_{k=1}^m$ just need to be convex and smooth, and the convex penalty $g$ is allowed to be non-smooth, e.g. the $\ell_1$ norm. On the contrary, most algorithms for distributed statistical estimation are only designed for smooth problems, and many of them are only rigorously studied when the loss functions are quadratic or self-concordant \citep{SSZ14,ZXi15,WWS17}.  This is another important aspect of our contributions.

We immediately see from Theorem \ref{thm-main} that Algorithms \ref{alg-new-single} and \ref{alg-new} converge linearly as long as $[ \delta / ( \rho_0+\alpha ) ]^{2} < \rho / ( \rho+2\alpha )$, which is guaranteed to hold by choosing sufficiently large $\alpha$. By contrast, however, we'll show in Section \ref{sec-vanilla-deterministic} that Algorithms \ref{alg-CSL} and \ref{alg-vanilla} (corresponding to $\alpha = 0$) hinges on the homogeneity assumption $\rho_0 > \delta$ in Theorem \ref{thm-vanilla}, i.e. the functions $\{ f_k \}_{k=1}^m$ must be similar enough. In the statistical setting, this requires the local sample size $n$ to be large. Therefore, proper regularization provides a safety net for the algorithms under general regimes with potentially insufficient local sample size. Corollary \ref{cor-main-alpha} below gives a guideline for choosing $\alpha$ to make Algorithms \ref{alg-new-single} and \ref{alg-new} converge in general.

\begin{cor}\label{cor-main-alpha}
	Let Assumptions \ref{assump-cvxity} and \ref{assump-similarity} hold, $\btheta_0 \in B(\widehat\btheta,R/2)$, and $\{ \btheta_t \}_{t=0}^{T}$ be the iterates of Algorithm \ref{alg-new-single} or \ref{alg-new}. With any $\alpha \geq 4 \delta^2/\rho$, both algorithms converge with contraction factors in (\ref{ineq-thm-main-1}) and (\ref{ineq-thm-main-2}) bounded by $( 1 - \frac{\rho}{ 10 ( \alpha + \rho ) } )$. 
\end{cor}

On the other hand, consider the case where the local loss functions have small relative difference $\delta/\rho$. In this case, Theorem \ref{thm-vanilla} states that the contraction factors for unregularized versions ($\alpha=0$) of Algorithms \ref{alg-new-single} and \ref{alg-new} are in the same order of $\delta/\rho$ and $(\delta/ \rho)^2$, respectively, which are smaller than the contracting factors with $\alpha > 0$.  The following corollary characterizes the upper bound for $\alpha$ so that the contraction factors remain at these small orders.

\begin{cor}\label{cor-main-similar}
	Let Assumptions \ref{assump-cvxity} and \ref{assump-similarity} hold, $\btheta_0 \in B(\widehat\btheta,R/2)$, and suppose $\alpha \leq C\delta^2/\rho$ for some constant $C$. There exist constants $C_1$ and $C_2$ such that the followings hold when $\delta / \rho$ is sufficiently small:
	\begin{itemize}
		\item [(a)]Algorithms \ref{alg-new-single} and \ref{alg-new} have the contraction property
		\begin{align*}
			\| \btheta_{t+1} - \widehat\btheta \|_2 \leq C_1 (\delta/\rho) \| \btheta_{t} - \widehat\btheta\|_2,\qquad 0 \leq t \leq T-1;
		\end{align*}
		\item [(b)] If Assumption \ref{assump-hessian-lip} also holds and $\|\btheta_t - \widehat\btheta\|_2 \leq \rho/M$, then
		for Algorithm \ref{alg-new}
		\begin{align*}
			\| \btheta_{t+1} - \widehat\btheta \|_2 \leq C_2 (\delta/\rho)^2 \| \btheta_{t} - \widehat\btheta\|_2.
		\end{align*}
		
	\end{itemize}
\end{cor}

Note that the second result above only holds given Assumption \ref{assump-hessian-lip}, which requires a smooth regularization $g$. Corollary \ref{cor-main-similar} reveals that  by choosing  $\alpha \asymp  \delta^2/\rho$, the contraction factors are essentially the same as those of the unregularized ($\alpha=0$) algorithms when  $\delta/\rho$ is small. By combining Corollaries \ref{cor-main-alpha} and \ref{cor-main-similar}, we use $\alpha \asymp \delta^2/\rho$ as a default choice for Algorithms \ref{alg-new-single} and \ref{alg-new} to become both fast and robust. They are reliable in general cases (Corollary \ref{cor-main-alpha}) and efficient in nice cases (Corolary \ref{cor-main-similar} and Theorem~\ref{thm-main} with $\alpha = 0$)

Algorithms \ref{alg-new-single} and \ref{alg-new} attain communication efficiency by utilizing similarity among local loss functions: The contraction factors in Corollary \ref{cor-main-similar} go to zero if $\delta/\rho$ does. In fact, both algorithms achieve $\varepsilon$-accuracy within
$
O( \max \{ 1, (\delta/\rho)^2 \} \log(\frac{\| \btheta_0 - \widehat\btheta\|_2}{\varepsilon}) )
$
rounds of communication. In contrast, the distributed accelerated gradient descent requires $O( \sqrt{ \kappa_0 } \log ( \frac{\| \btheta_0 - \widehat\btheta\|_2}{\varepsilon} ) )$ rounds of communication to achieve $\varepsilon$-accuracy \citep{SSZ14}, with $\kappa_0$ being the condition number of $(f+g)$, which does not take advantage of sample size $n$. As long as $ \delta /\rho \ll \kappa_0^{1/4} $, Algorithms \ref{alg-new-single} and \ref{alg-new} communicate less than the distributed accelerated gradient descent. And again, our general results for Algorithms \ref{alg-new-single} and \ref{alg-new} also apply to the case with nonsmooth penalty functions while those for distributed accelerated gradient descent do not.

Moreover, if $(f+g)$ is smooth and $\btheta_t$ is reasonably close to $\widehat\btheta$, Corollary \ref{cor-main-similar} shows that each iteration of Algorithm \ref{alg-new} is roughly equivalent to two iterations of Algorithm \ref{alg-new-single}, although the former only has one round of optimization. The averaging step in Algorithm \ref{alg-new} reduces the error as much as the optimization step, while taking much less time. In this case, Algorithm \ref{alg-new} is preferable, and our numerical experiments also confirm this.

For unregularized empirical risk minimization, i.e. $g=0$ in (\ref{eqn-distributed}), Algorithm \ref{alg-new} reduces to an extension or a useful case of the DANE algorithm \citep{SSZ14}. 
In this case, Theorem \ref{thm-main} and its corollaries refine the analysis of DANE \citep{SSZ14} in several aspects. On the one hand, our analysis handle both smooth and nonsmooth problems, while in \cite{SSZ14}, the theoretical analysis beyond quadratic loss requires extremal choice of tuning parameters and does not show any advantage over distributed implementation of the gradient descent. On the other hand, as mentioned in Section \ref{sec-setup}, we derive Algorithm \ref{alg-new} from the proximal point algorithm with a new prospective, which leads to sharp convergence analysis along with suggestions on choosing the tuning parameter $\alpha$. As a by-product, we close a gap in the theory of DANE in non-quadratic settings. Our analysis techniques are potentially useful for other distributed optimization algorithms, especially when the loss is not quadratic.

\subsection{Deterministic analysis in large-sample regimes}\label{sec-vanilla-deterministic}


In this section, we restrict ourselves to large-sample regimes where the local sample size $n$ is sufficiently large such that $\rho_0 > \delta \geq 0$, where $\rho_0$ is the strong convexity parameter in (\ref{eqn-rho0}). The following theorem gives deterministic results for Algorithms \ref{alg-CSL} and \ref{alg-vanilla}.

\begin{thm}\label{thm-vanilla}
	Let Assumptions \ref{assump-cvxity} and \ref{assump-similarity} hold, and $\rho_0 > \delta \geq 0$. Consider the iterates $\{ \btheta_t \}_{t=0}^{\infty}$ produced by Algorithm \ref{alg-CSL} or \ref{alg-vanilla}, with $\btheta_0 \in B( \widehat\btheta, R)$. Then
	\[
	\| \btheta_{t+1}-\widehat\btheta\|_2 \leq (\delta/\rho_0) \|\btheta_{t}-\widehat\btheta\|_2,\qquad\forall t\geq 0.
	\]
	In addition, if Assumption \ref{assump-hessian-lip} also holds, then for Algorithm \ref{alg-vanilla} we have
	\begin{align*}
		\| \btheta_{t+1}-\widehat\btheta\|_2 \leq \frac{\delta}{\rho_0} \|\btheta_{t}-\widehat\btheta\|_2 \cdot \min\left\{ 1 , \frac{\delta}{\rho} \left(1+ \frac{M}{\rho_0} \| \btheta_t - \widehat\btheta\|_2 \right) \right\},\qquad \forall t \geq 0.
	\end{align*}
\end{thm}

Note that the last inequality requires Assumption \ref{assump-hessian-lip} and thus a smooth regularization $g$. The first part of Theorem \ref{thm-vanilla} is a refinement of the analysis in \cite{JLY18}, since we allow the initial estimator to be inaccurate and we have more explicit rates of contraction of optimization errors. This will be further demonstrated in Section \ref{section-vanilla-glm}. 

The second part points out benefits of the averaging step, which is a novel result. Similar to the results on Algorithm \ref{alg-new}, Theorem \ref{thm-vanilla} shows that when $\btheta_t$ is close to $\hat \btheta$, with an additional standard assumption on Hessian smoothness, the averaging step alone in Algorithm \ref{alg-vanilla} is almost as powerful as an optimization step in terms of contraction: The contracting constant will eventually be $\frac{\delta}{\rho_0} \frac{\delta}{\rho}$. With negligible computational cost, averaging significantly improves upon individual solutions $\{ \btheta_{t,k} \}_{k=1}^m$ by doubling the speed of convergence. 

\section{Statistical analysis}\label{sec-general}

We further analyze the statistical properties of the above algorithms under a generalized linear model. Essentially, both the CSL methods and the CEASE algorithm are $T$-step estimators, starting from the initial estimator $\btheta_0$. The question here is the effect of iterations in the multiple step estimators and the role of the initial estimator. We start with statistical analysis of Algorithms \ref{alg-new-single} and \ref{alg-new} in Section \ref{sec-CEASE-stat}, and then study Algorithms \ref{alg-CSL} and \ref{alg-vanilla} in Section \ref{section-vanilla-glm}. In Section \ref{sec-practice}, we provide practical guidance when implementing the CEASE algorithm based on these analysis.

\subsection{Multi-step estimators in general regimes}\label{sec-CEASE-stat}

The deterministic analysis in Section \ref{sec-vanilla-deterministic} applies to a wide range of statistical models. Here we consider the generalized linear model with canonical link, where our samples are i.i.d. pairs $\{(\bx_i,y_i)\}_{i=1}^{N}$ of covariates and responses and the conditional density of $y_i$ given $\bx_i$ is given by
$$
h(y_i;\bx_i,\btheta^*)=c(\bx_i, y_i)\exp\left(\y_i(\bx_{i}^\top\btheta^*)-b(\bx_{i}^\top\btheta^*)\right).
$$
For simplicity, we let the dispersion parameter to be 1 as we do not consider the issue of over-dispersion;  $b(\cdot)$ is some known convex function, and $c$ is a known function such that $h$ is a valid probability density function.
The negative log-likelihood of the whole data is an affine transformation of $f(\btheta)=\frac{1}{m}\sum_{k=1}^mf_k(\btheta)$ with
$$f_k(\btheta)=\frac{1}{n}\sum_{i\in\mathcal I_k} \left[b(\bx_{i}^\top\btheta)-y_i(\bx_{i}^\top\btheta)\right].$$
It is easy to verify that
$$
\nabla f_k(\btheta)=\frac{1}{n}\sum_{i\in\mathcal I_k}[b'(\bx_{i}^\top\btheta) - y_i ] \bx_{i}
\quad \mbox{and} \quad
\nabla^2 f_k(\btheta)=\frac{1}{n}\sum_{i\in\mathcal I_k}b''(\bx_{i}^\top\btheta)\bx_{i}\bx_{i}^\top.
$$

Assume that $\bx_{i}=(1,\bu_{i}^{\top})^{\top} \in \R^p$, where $\{\bu_i\}_{i=1}^{N} \subseteq \R^{p-1}$ are i.i.d. random covariate vectors with zero mean and covariance matrix $\bSigma$. Suppose there exist universal positive constants $A_1$, $A_2$ and $A_3$ such that $A_1\leq \ltwonorm{\bSigma }\leq A_2 p^{A_3}$. Let $\bSigma^*=\E ( \bx_{i}\bx_{i}^\top ) = \begin{pmatrix}
1 & \mathbf{0} \\
\mathbf{0} & \bSigma
\end{pmatrix}
$, $g$ be a deterministic penalty function, and $F(\btheta) = \E f(\btheta)$ be the population risk function. Below we impose some standard regularity assumptions.

\begin{ass}\label{assump-glm-subg-b2-b3}
	
	\begin{itemize}
		\item  $\{ \bSigma^{-1/2} \bu_{i} \}_{i=1}^{N}$ are i.i.d. sub-Gaussian random vectors.
		\item  For all $x\in\R$, $|b''(x)|$ and $|b'''(x)|$ are bounded by some constant.
		\item  $\| \btheta^* \|_2$ is bounded by some constant.
	\end{itemize}
\end{ass}

As in Assumptions \ref{assump-cvxity} and \ref{assump-hessian-lip}, the following general assumptions are also needed for our analysis. Here $R$ is some positive quantity that satisfies $R<A_4 p^{A_5}$ for some universal constants $A_4$ and $A_5$.

\begin{ass}\label{assump-cvxity-stat}
	There exists a universal constant $\rho>0$ such that  $(F+g)$ is $\rho$-strongly convex in $B(\btheta^*, 2R)$.
\end{ass}

The following smoothness assumption is only needed for a part of our theory; it is used to show that the averaging step in Algorithm \ref{alg-new} can significantly enhance the accuracy.

\begin{ass}\label{assump-hlip-stat}
	$g\in C^2(\R^p)$, and there exists a universal constant $M\geq 0$ such that  $\| [ \nabla^2 F(\btheta') + \nabla^2 g(\btheta') ] - [ \nabla^2 F(\btheta'') + \nabla^2 g(\btheta'') ] \|_2 \leq M \| \btheta'-\btheta''\|_2$, $\forall \btheta',\btheta'' \in B(\btheta^*,2R)$.
\end{ass}

Under the model assumptions above, we can explicitly determine the rate of $\delta$ in  Assumption \ref{assump-similarity}. In particular, we will show in Lemma E.5 in Appendix E that
$$
\max_{k\in[m]}\max_{\btheta\in{B(\widehat\btheta, R)}}\| \nabla^2 f_k(\btheta) - \nabla^2 f (\btheta) \|_2 =O_{\mathbb{P}} \left( \ltwonorm{\bSigma}\sqrt{\frac{ p(\log p+\log N)}{n}}\right),
$$
provided that $n\geq c p$ for an arbitrary positive constant $c$.
Therefore, with high probability,  $\delta \asymp \ltwonorm{\bSigma}\sqrt{ p(\log p+\log N)/{n}}$. Omitting the logarithmic terms, we see that the contraction factor is approximately $\kappa\sqrt{p/n}$, where $\kappa\triangleq \ltwonorm{\bSigma}/\rho$ can be viewed as a condition number.  This rate is more explicit on $p$ and $\kappa$ than that in \cite{JLY18}, where finite $p$ and $\kappa$ are assumed.
In addition, with a smooth regularization, Algorithm \ref{alg-new} benefits from the averaging step in that it improves the contraction rate to approximately $\kappa^2p/n$.

Let $\btheta_t$ be the $t$-th iterate of one of the proposed algorithms.
It is clear that the statistical error of the estimator $\btheta_t$ is bounded by its optimization error and the statistical error of $\widehat\btheta$:
$$
\| \btheta_t - \btheta^*\|_2 \leq  \| \btheta_t - \widehat \btheta\|_2 +  \| \widehat \btheta - \btheta^*\|_2.
$$
The second term is well-studied in statistics, which is of order $O_{\P}( \sqrt{p/N})$ under mild conditions. In the following theorem, we show that for the first term (i.e. the optimization error), with proper choice of $\alpha$, each iteration of Algorithm \ref{alg-new-single} or \ref{alg-new} makes $\btheta_t$ closer to the global minimum $\widehat \btheta$ by some order depending on local sample size. Thus, through finite steps, the optimization errors are eventually negligible in comparison with statistical errors (assuming $N$ is of order $(n/p)^a$ for a finite $a$ in typical applications), and the distributed multi-step estimator will work as well as the global minimum as if the data were aggregated in the central server.

\begin{thm}\label{thm-glm-alpha}
	Suppose that Assumptions \ref{assump-glm-subg-b2-b3} and \ref{assump-cvxity-stat} hold, and with high probability the initial value $\btheta_0\in B(\widehat\btheta, R/2)$.
	Let $\eta=\kappa^2(\log N)p/n$ and $\kappa=\ltwonorm{\bSigma}/\rho$.
	For any $c_1,c_2>0$, there exists $C>0$ such that the followings hold with high probability:
	\begin{itemize}
		\item [(a)]  If $n\geq c_1 p$ and $\alpha \geq C \rho\eta$, then both Algorithms \ref{alg-new-single} and \ref{alg-new} have linear convergence
		$$
		\| \btheta_{t} - \widehat\btheta \|_2 \leq \left[ 1 - \frac{\rho}{ 10 ( \alpha + \rho ) } \right]^t \| \btheta_0 - \widehat \btheta \|_2,\qquad \forall t \geq 0;
		$$
		
		\item [(b)]If $\eta$ is sufficiently small and $\alpha \leq c_2 \rho\eta$, then for both algorithms
		$$
		\| \btheta_{t} - \widehat\btheta \|_2 = O_{\P}( \eta^{t/2} \| \btheta_0 - \widehat \btheta \|_2 ),\qquad \forall t \geq 0;
		$$
		in addition, if Assumption \ref{assump-hlip-stat} also holds, then for Algorithm \ref{alg-new}  we have
		$$
		\| \btheta_{t} - \widehat\btheta \|_2 = O_{\P} ( \eta^{t - t_0} \| \btheta_{t_0} - \widehat\btheta\|_2 ), \qquad \forall t \geq t_0,
		$$
		where $t_0 = \lceil \frac{2 \log (CMR/\rho) }{ \log(1/\eta) } \rceil $.
	\end{itemize}
\end{thm}

In contrast to a fixed contraction derived by \cite{SSZ14}, Theorem \ref{thm-glm-alpha} explains the significant benefits of large local sample size even in the presence of a non-smooth penalty: the optimization error shrinks by a factor that converges to zero explicitly in $n$. As a brief illustration, let us consider the case with smooth loss functions, sufficient local sample size and no regularization. Let $\btheta_0$ be the average of individual estimators on node machines. By Corollary 2 in \cite{ZDW13}, this simple divide-and-conquer estimator has accuracy $\ltwonorm{ \btheta_0 -\btheta^*}=O_{\P}( \max \{\sqrt{\frac{p}{N}},\frac{p}{n},\frac{\kappa \sqrt{p\log p}}{n} \} )$.  Using the explicit expression of $\eta$, we can easily deduce from Theorem~\ref{thm-glm-alpha} (b)  that the one-step estimator $\btheta_1$ obtained by Algorithm \ref{alg-new-single} behaves the same as the global minimizer $\widehat \btheta$ if the local sample size satisfies
$
	n^3 \gg N (\kappa^2 p \log N) (p + \kappa^2 \log p).
$	
In this case, the local optimization in Algorithm~\ref{alg-new-single} can further be replaced by using the explicit one-step estimator as in \cite{Bic75} and \cite{JLY18}, since the initial estimator is in a consistent neighborhood. More generally, the $t$-step estimator $\btheta_t$ has negligible optimization error under even weaker local sample size requirement:
$ 
  n^{t+2} \gg N( \kappa^2 p \log N)^t (p + \kappa^2 \log p). \label{eq2.6}
$
 A similar remark applies to Algorithm \ref{alg-new}.


Similar to the deterministic results, the averaging step is about as effective as the optimization step when $g$ is smooth and $\btheta_t$ is sufficiently close to $\hat\btheta$ in that after a finite $t_0$ iterations. See a simplified example in Section \ref{section-vanilla-glm} with $\alpha = 0$.

As for the initialization, the condition $\btheta_0 \in B(\widehat\btheta,R/2)$ is mild, since $\widehat\btheta$ is usually a consistent estimate and  $\| \btheta^* \|_2$ is bounded (Assumption \ref{assump-glm-subg-b2-b3}). In contrast with \cite{JLY18}, we allow inaccurate initial value such as $\btheta_0= \mathbf{0}$ and give more explicit rates of contraction even when $p$ and $\kappa$ diverge. On the other hand, the accuracy of the initial estimator $\btheta_0$ does help reduce the number of iterations.  

Combining the results to be presented in Theorem \ref{thm-vanilla-glm}, we'll see that by choosing $\alpha\asymp \rho\eta$, Algorithms \ref{alg-new-single} and \ref{alg-new} inherit all the merits of Algorithms \ref{alg-CSL} and \ref{alg-vanilla} in the large-$n$ regime -- fast linear contraction of rate $\sqrt{\eta}=\kappa\sqrt{p(\log N)/n}$, and for Algorithm \ref{alg-new}, a even faster rate of $\eta=\kappa^2p(\log N)/n$ to $\widehat\btheta$ when the loss  and the penalty functions are smooth. These facts also guarantee that Algorithms \ref{alg-new-single} and \ref{alg-new} reach the statistical efficiency in $O(\frac{\log \ltwonorm{\btheta_0-\widehat\btheta} + \log (N/p)}{\log (1/\eta)})$ iterations. On the other hand, compared to Algorithms \ref{alg-CSL} and \ref{alg-vanilla}, Algorithms \ref{alg-new-single} and \ref{alg-new} overcome the difficulties with a small local sample size $n$ in that as long as $n/p$ is bounded away from some small constant (which is reasonable for many big-data problems of interest), shrinkage of optimization error is guaranteed. Moreover, while it is hard to check whether $n$ is sufficiently large in practice, proper choice of $\alpha$ always guarantees linear convergence, and the contraction rates adapt to the sample size $n$. In this way, Algorithms \ref{alg-new-single} and \ref{alg-new} perfectly resolve the main issue of their vanilla versions.

We can get stronger results in the specific case of distributed linear regression, where the contraction rate has nearly no dependence on the conditional number $\kappa$. Due to space constraints, we put all the details in Appendix B.

\subsection{Multi-step estimators in large-sample regimes}\label{section-vanilla-glm}

We now present the contraction of optimization error of Algorithms \ref{alg-CSL} and \ref{alg-vanilla}. 

\begin{thm}\label{thm-vanilla-glm}
	Suppose that Assumptions \ref{assump-glm-subg-b2-b3} and \ref{assump-cvxity-stat} hold, and with probability tending to one, $\btheta_0 \in B(\widehat\btheta,R)$ for some $R>\ltwonorm{\widehat\btheta-\btheta^*}$. For Algorithms \ref{alg-CSL} and \ref{alg-vanilla}, we have
	$$
	\| \btheta_{t} - \widehat\btheta \|_2   = O_{\P} ( \eta^{t/2} \| \btheta_0 - \widehat \btheta \|_2 ),\qquad \forall t \geq 0,
	$$
	where $\eta = \kappa^2 p (\log N)/n$. In addition, let Assumption \ref{assump-hlip-stat} also hold. There exists some constant $C$ such that for Algorithm \ref{alg-vanilla} we have
	\begin{align}
		\| \btheta_{t} - \widehat\btheta \|_2 = O_{\P} ( \eta^{t - t_0} \| \btheta_{t_0} - \widehat\btheta \|_2 ),\qquad \forall t \geq t_0,\label{eq-vanilla-stat-2}
	\end{align}
	where $t_0 = \lceil \frac{ 2 \log ( C M R / \rho ) }{ \log (1/\eta) } \rceil$.
\end{thm}

The strengthened result \eqref{eq-vanilla-stat-2} requires Assumption \ref{assump-hlip-stat} and thus smooth $g$. Theorem \ref{thm-vanilla-glm} shows that when $n$ is sufficiently large, Algorithms \ref{alg-CSL} and \ref{alg-vanilla} behave similarly as Algorithms \ref{alg-new-single} and \ref{alg-new} -- faster convergence with the larger $n$, mild restrictions on initialization, and averaging speeds up contraction given smooth loss. However, there is no convergence guarantee in general regimes. Section \ref{sec-numerical} further shows that with insufficient local sample size, the practical performance Algorithms \ref{alg-CSL} and \ref{alg-vanilla} is less satisfactory. 

Finally, to see how the averaging step reduces the statistical error (i.e. the distance between the estimator and $\btheta^*$), 
we continue to look at the linear regression example mentioned at the end of Section \ref{section2-2}. For simplicity, assume $\{\bx_i\}_{i\in[N]}$ are i.i.d. standard normal random vectors. \eqref{eqn-dane1} and \eqref{eqn-dane2} can be expressed as
\begin{align*}
\widehat\bSigma^{1/2}(\btheta_{t+1, k}  - \widehat\btheta )&= \left(\bI - \widehat\bSigma^{1/2}\widehat\bSigma_k^{-1}\widehat\bSigma^{1/2}\right)\cdot \widehat\bSigma^{1/2}(\btheta_{t}  - \widehat\btheta ),\\
\widehat\bSigma^{1/2}(\btheta_{t+1}  - \widehat\btheta )&= \left[\bI - \widehat\bSigma^{1/2}\bigg(\frac{1}{m}\sum_{k=1}^m \widehat\bSigma_k^{-1}\bigg)\widehat\bSigma^{1/2}\right]\cdot \widehat\bSigma^{1/2}(\btheta_{t}  - \widehat\btheta ).
\end{align*}
If further $n \gg p$, then the two contraction factors satisfy
\begin{align}
&\| \bI - \widehat\bSigma^{1/2}\widehat\bSigma_k^{-1}\widehat\bSigma^{1/2}\|_2 = O_{\mathbb P}(\sqrt{p/n}),
\label{eqn-contraction-1}
\\
&\bigg\|\bI - \widehat\bSigma^{1/2}\bigg(\frac{1}{m}\sum_{k=1}^m \widehat\bSigma_k^{-1}\bigg)\widehat\bSigma^{1/2}\bigg\|_2  = O_{\mathbb P}(p/n).
\label{eqn-contraction-2}
\end{align}
When the Algorithm~\ref{alg-CSL} applies to this problem, the expression changes slightly to
$$
\widehat\bSigma^{1/2}(\btheta_{t+1, 1}  - \widehat\btheta ) = \left(\bI - \widehat\bSigma^{1/2}\widehat\bSigma_1^{-1}\widehat\bSigma^{1/2}\right)\cdot \widehat\bSigma^{1/2}(\btheta_{t, 1}  - \widehat\btheta ).
$$
The smaller magnitude in \eqref{eqn-contraction-2} is due to the averaging.

Suppose that we initialize the algorithm using the one-shot average $\btheta_0 = \frac{1}{m} \sum_{k=1}^{m} \widehat\btheta_{k}$, where $\widehat\btheta_{k}$ is the least squares solution on the $k$th machine. When $n = O(\sqrt{Np})$, we have $ \sqrt{p/N} = O(p/n )$. \cite{ZDW13} assert that $\| \widehat\btheta - \btheta^* \|_2 = O_{\mathbb P}(\sqrt{\frac{p}{N}}) $, $\| \btheta_0 - \widehat\btheta \|_2 = O_{\PP} (\frac{p}{n})$. Thus, the initial statistical error is $\| \btheta_0 - \btheta^* \|_2 = O_{\PP} (
\max\{
\sqrt{\frac{p}{N}}, \frac{p}{n}
\}
 ) = O_{\PP} (\frac{p}{n})$.
By the contraction properties in (\ref{eqn-contraction-1}) and (\ref{eqn-contraction-2}), the optimization errors are
\begin{align*}
& \|\btheta_{1,k} - \widehat\btheta \|_2 = O_{\mathbb P}(\sqrt{p / n}) \|\btheta_{0} - \widehat\btheta \|_2 = O_{\mathbb P}( p^{3/2} / n^{3/2} ), \\
& \|\btheta_{1} - \widehat\btheta \|_2 = O_{\mathbb P}(p / n) \|\btheta_{0} - \widehat\btheta \|_2 = O_{\mathbb P}( p^2 / n^2 ).
\end{align*}
We can see that when $p^2 / n^2 \ll \sqrt{ p / N} \ll p^{3/2} / n^{3/2}$ (or equivalently, $ p^{3/4} N^{1/4}  \ll  n \ll p^{3/2} N^{1/3}$), the optimization error is negligible for $\btheta_{1}$, but is not negligible for $\btheta_{1, k}$. 
When $n$ is smaller, even more iterations are needed. A refined analysis of distributed least squares is in Appendix B.

\subsection{Guidance on practice}\label{sec-practice}

We now provide some general guidance on how to implement the CEASE algorithm in practice. First, we recommend choosing an initialization depending on the magnitude of the local sample size $n$. In particular, when $n$ is not very large compared to the dimension $p$, a zero initialization would be more robust. On the other hand, given a moderate or large $n$, the one-shot average estimator will lead to extremely fast convergence.

Second, according to Theorem \ref{thm-glm-alpha}, it suffices to take $\alpha$ of the order of $\rho \kappa^2 p\log N / n$. In practice, setting $\alpha$ to be a small multiple of $p/ n$ seems suitable in many occasions. 

Finally, as is already shown, both Algorithm \ref{alg-new-single} and \ref{alg-new} reach statistical efficiency in $O(\frac{\log \ltwonorm{\btheta_0-\widehat\btheta} + \log (N/p)}{\log (1/\eta)})$ iterations. In both our simulations and real data examples, the CEASE algorithms with a properly chosen $\alpha$ converge to the centralized estimator within 10 iterations. With a moderate $n$, a warm start further boosts the convergence speed.

\section{Numerical experiments}\label{sec-numerical}
\subsection{Synthetic data}\label{sec-syn}

We first conduct distributed logistic regression to illustrate the effect of local sample size and initialization on convergence. We keep the total sample size $N = 10000$ and the dimensionality $p = 101$ fixed, and generate the i.i.d. data $\{ ( \bx_i, y_i ) \}_{i=1}^N $ as follows: $\bx_i = (1,\bu_i^{\top})^{\top}$ with $\bu_i \sim N( \mathbf{0}_{p-1} , \bSigma )$ and $\bSigma = \diag ( 10,5,2,1\cdots 1 ) \in \R^{(p-1)\times (p-1)}$; $\P(y_i=1) = 1 - \P(y_i = 0) = 1 / (1 + e^{ -\bx_i^{\top}\btheta^* }) $ where $\btheta^* \in \R^{p}$ is a random vector with norm 3 whose direction is chosen uniformly at random from the sphere. We use the natural logarithm of the estimation error $\| \btheta_t - \btheta^* \|_2$ to measure the performance of different algorithms, including multiple versions of the CEASE algorithms, GIANT \citep{WRX18}, ADMM \citep{BPC11} and accelerated gradient descent \citep{Nes83}.

Figure \ref{fig-log} shows how the estimation errors evolve with iterations. The curves show the average values over 100 independent runs; the error bands correspond to one standard deviation. The regimes ``large $n$", ``moderate $n$" and ``small $n$" refer to $(n,m)=(2000,5)$, $(1000,10)$ and $(250,40)$; ``zero initialization" and ``good initialization" refer to $\btheta_0 = \mathbf{0}$ (bottom panel) and $\bar{\btheta}$ (top panel), respectively. Here $\bar{\btheta}$ is the one-shot distributed estimator \citep{ZDW13} that averages the individual estimators on node machines.
	According to Figure \ref{fig-log}, the standard deviation of each iterate is around 0.1. For the ``large $n$, good initialization'' regime, all of the iterates are unsurprisingly very close to the optimal solution. Their error bands will cover up the curves. So we omit the bands in that case for the sake of clarity.

With proper regularization, the two CEASE algorithms are the only ones that converge rapidly in all scenarios. The purely deterministic methods ADMM \citep{BPC11} and accelerated gradient descent \citep{Nes83} are also reliable but slow. Other distributed algorithms like unregularized CEASE and GIANT \citep{WRX18} easily fail when the local sample size is small or the initialization is uninformative. In addition, the CEASE with averaging (Algorithm \ref{alg-new}) is superior to the one without averaging (Algorithm \ref{alg-new-single}). For example, when $(n,m)=(1000,10)$, the averaged CEASE with $\alpha = 0$ converges while the one without averaging does not. Hence the averaging step leads to better performance.

\begin{figure}[ht]
	\centering
	\subfigure{
		\includegraphics[width=1\linewidth]{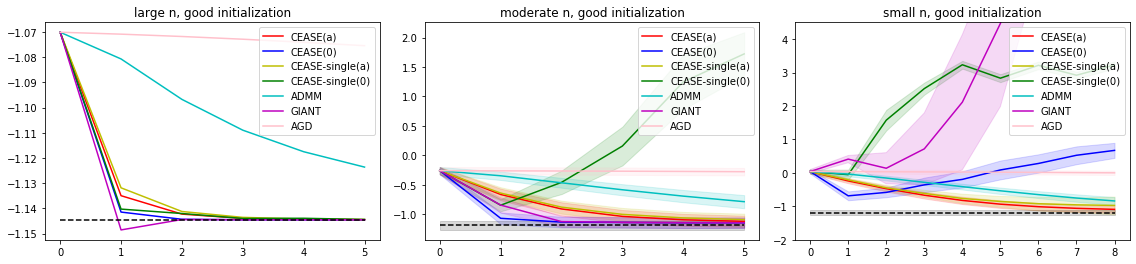}  
	}
	\\
	\subfigure{
		\includegraphics[width=1\linewidth]{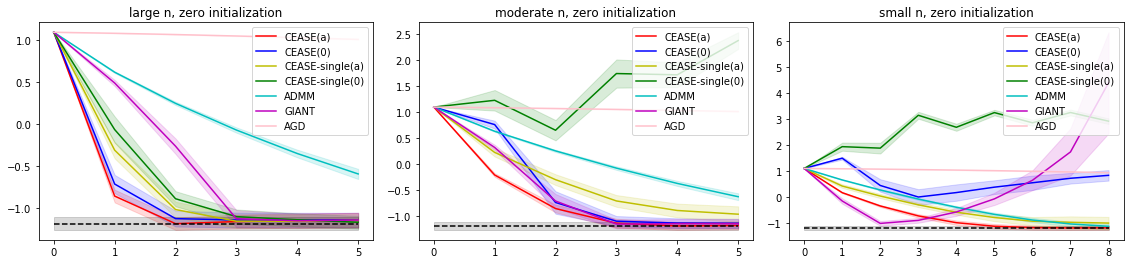}    
	}
	\caption{Impacts of local sample size and initialization on convergence. The $x$-axis and $y$-axis are the number of iterations and $\log \| \btheta_t - \btheta^* \|_2$. The dashed lines show the error of the minimizer of the overall loss function. The top and bottom panels use $\bar {\btheta}$ and $\mathbf{0}$ for initialization, respectively. CEASE(a) and CEASE(0) refer to Algorithm \ref{alg-new} with $\alpha=0.15 p/n$ and $0$; CEASE-single(a) and CEASE-single(0) refer to Algorithm \ref{alg-new-single} with $\alpha=0.15 p/n$ and $0$, respectively. In particular, CEASE-single(0) is equivalent to the CSL algorithm in \cite{JLY18}.}\label{fig-log}
\end{figure}

We also test the efficacy of our algorithms in the distributed $\ell_1$-regularized logistic regression, where the penalty $g$ is nonsmooth (See Appendix D for details). To summarize, our simulations demonstrate several important properties of the CEASE Algorithms:

\begin{itemize}
	\item In all scenarios, the CEASE Algorithms converge rapidly, usually within several steps, which is consistent with our theory;
	\item The CEASE Algorithms efficiently utilize statistical structures and similarities among local losses, and benefit from the averaging step with smooth loss functions;
	\item The CEASE Algorithms are also able to handle the most general situations (e.g. small local sample size, uninformative initialization) with convergence guarantees.
	
\end{itemize}

\subsection{Real data}

	As a real data example, we choose the Fashion-MNIST dataset \citep{XRV17} as a testbed for comparison of algorithms. The whole dataset consists of 70000 grayscale images of fashion products in 10 classes, each of which has 6000 training samples and 1000 testing samples. We choose the 7th and 9th classes (Sneakers and Ankle boots) and the goal is to train a classifier that distinguishes them. Each image has $28\times 28 = 784$ pixels, represented by a feature vector in $[0, 1]^{784}$. The number of training (or testing) samples is $6000 \times 2 = 12000$ (or $1000 \times 2 = 2000$). We randomly partition the training set and conduct logistic regression in a distributed manner. The performance metric is the classification error on the testing set. Figure \ref{fig-spam} shows the average performance of the CEASE algorithms, ADMM, GIANT and AGD based on 100 independent runs, together with error bars showing one standard deviation. Here ``large $n$", ``moderate $n$" and ``small $n$" refer to $(n,m)=(1200,10)$, $(480, 25)$ and $(240,50)$, respectively. All of the iterations are initialized with the one-shot average \citep{ZDW13}. The experiments on this real data example also support our theoretical findings.

	\begin{figure}[ht]
		\centering
		\includegraphics[width=1\linewidth]{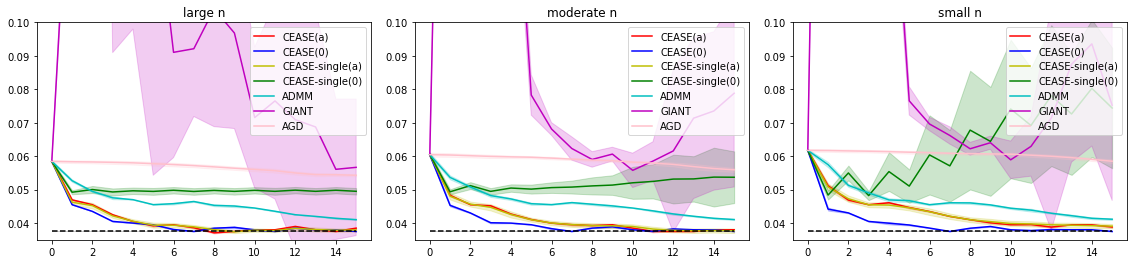}  
		\caption{Fashion-MNIST dataset. The $x$-axis and $y$-axis are the number of iterations and the testing error.  The dashed lines show the error of the classifier based on all of the training samples.   All of the iterations are initialized with the one-shot average $\bar{\btheta}$.
			CEASE(a) and CEASE(0) refer to Algorithm \ref{alg-new} with $\alpha=0.15 p/n$ and $0$; CEASE-single(a) and CEASE-single(0) refer to Algorithm \ref{alg-new-single} with $\alpha=0.15 p/n$ and $0$, respectively. In particular, CEASE-single(0) is equivalent to the CSL algorithm in \cite{JLY18}. GIANT and CEASE-single(0) do not converge to the optimal solution. }\label{fig-spam}
	\end{figure}

\section{Discussions}\label{sec-discussion}
We have developed two CEASE distributed estimators (Algorithms \ref{alg-new-single} and \ref{alg-new}) for statistical estimation, with theoretical guarantees and superior performance on real data. Several new directions are worth exploring. First, while we assumed exact computation for simplicity, finer analysis should allow for inexact updates in practice. Second, we hope to extend the algorithms to decentralized and asynchronous settings. Third, distributed versions of confidence regions and hypothesis tests are of great importance, and our point estimation strategies may serve as a starting point. Finally, it will be interesting to explore non-convex statistical optimization problems such as mixture models and deep learning.  We believe that the idea of gradient-enhanced loss function still plays an important role.

\section*{Acknowledgement}

We gratefully acknowledge NSF grants DMS-1662139 and  DMS-1712591, NIH grant 2R01-GM072611-14, and ONR grant N00014-19-1-2120. We acknowledge computing resources from Columbia University's Shared Research Computing Facility project, which is supported by NIH Research Facility Improvement Grant 1G20-RR030893-01, and associated funds from the New York State Empire State Development, Division of Science Technology and Innovation (NYSTAR) Contract C090171, both awarded April 15, 2010.

\newpage 

\appendix

\section*{Appendices}

Section \ref{appendix-stepsize} introduces a variant of Algorithm 4. Section \ref{appendix-quadratic} outlines the deferred results for distributed linear regression. Section \ref{appendix-logistic} shows the CEASE iterates for distributed logistic regression.
Section \ref{appendix-l1_logistic_simulations} presents the numerical results on the distributed $\ell_1$-regularized logistic regression. Section \ref{appendix-proofs} presents the  proofs of the main results. Section \ref{appendix-lemmas} lists a few technical lemmas that are used throughout the proofs.

\section{A variant of Algorithm 4}\label{appendix-stepsize}

As mentioned in the main text, Algorithm 2 is unstable when the local sample size $n$ is not sufficiently large.  The proximal gradient method Algorithm \ref{alg-vanilla-variant} was introduced to stabilize the solution path by shrinking towards the solution in the previous step.  A variant of CEASE that stablizes Algorithm 2 is to take smaller step-sizes, which we now present.  The idea is applicable to stabilize Algorithm 1 too, resulting a variant to Algorithm 3.

\begin{algorithm*}[t]
	\caption{{Distributed estimation using gradient-enhanced loss (small step-sizes)}}	
	\label{alg-vanilla-variant}\begin{algorithmic}[1]		
		\STATE \textbf{{Input}}:
		Initial value $\btheta_{0}$, number of iterations $T$, step-sizes $\{ \alpha_t \}_{t=0}^{T - 1}$.
		\STATE \textbf{For} $t=0,1,2,\cdots,T-1$:
		\begin{itemize}
			\item Each machine evaluates $\nabla f_{k}(\btheta_{t})$ and sends to the
			central processor;
			\item The central processor computes $\nabla f(\btheta_{t})=\frac{1}{m} \sum_{k=1}^{m} \nabla f_{k}(\btheta_{t})$
			and broadcasts to machines;
			\item Each machine computes
			\begin{align*}
			\btheta_{t,k}=\argmin_{\btheta}\left\{ f_{k}(\btheta) + g(\btheta) - \langle \nabla f_{k}(\btheta_{t}) - \nabla f(\btheta_{t}) ,\btheta\rangle \right\}
			\end{align*}
			and sends to the central processor;
			\item The central processor computes $\btheta_{t+1} =
			(1 - \alpha_t) \btheta_t +
			\frac{ \alpha_t }{m}\sum_{k=1}^{m}\btheta_{t,k}$
			and broadcasts to machines.
		\end{itemize}
		\STATE \textbf{{Output}}: $\btheta_{T}$.
	\end{algorithmic}
\end{algorithm*}

The only difference between Algorithms 2 and \ref{alg-vanilla-variant} lies in the aggregation step. From $\btheta_t$ the former directly jumps to the average of new individual estimators $\{ \btheta_{t,k} \}_{k=1}^m$, while the latter proceeds more cautiously in that direction. Algorithm 2 is a special case of Algorithm \ref{alg-vanilla-variant} with $\alpha_t = 1$ for all $t$. Choosing $\alpha_t \in (0, 1)$ helps stablize the iterates especially when $n$ is small.

Algorithm \ref{alg-vanilla-variant} is conceptually simple and easy to implement. To see its performance, we conduct distributed logistic regression on the first set of synthetic data in Section 5.1. The numerical results there show that Algorithm 2 fails to converge when $(n, m) = (250, 40)$. We run Algorithm \ref{alg-vanilla-variant} with constant step-size $\alpha_0 = \alpha_1 = \cdots$ under this setting, where $\alpha_0 \in \{ 1, 1/2, 1/4, 1/8 \}$. We also run CEASE for comparison. Figure \ref{fig-stepsize} summarizes all the results. Again, the curves show the average values over 100 independent runs; the error bands correspond to one standard deviation. In this experiment, the performance of Algorithm \ref{alg-vanilla-variant} with $\alpha_t = 1/4$ is similar to that of CEASE; $\alpha_t = 1/2$ leads to even faster convergence; and $\alpha_t = 1/8$ slows it down. For simplicity, we take the step-size to be constant over time. It would be interesting to explore decaying schemes such as $\alpha_t \asymp t^{-\beta}$ for some $\beta > 0$.

	\begin{figure}[ht]
	\centering
	\includegraphics[width=0.9\linewidth]{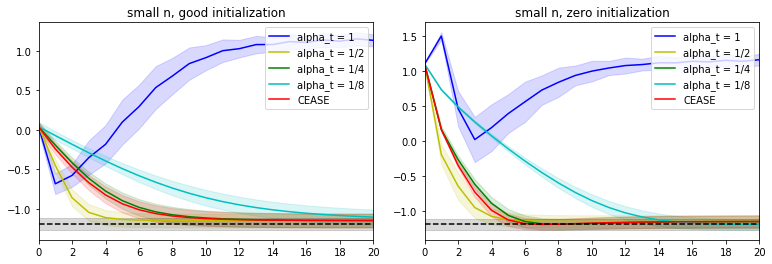}  
	\caption{Impacts of step-size $\alpha_t$ on convergence. The $x$-axis and $y$-axis are the number of iterations and $\log \| \btheta_t - \btheta^* \|_2$. The dashed lines show the error of the minimizer of the overall loss function. The left and right plots use $\bar {\btheta}$ and $\mathbf{0}$ for initialization. $\alpha_t = 1, 1/2, 1/4$ and $1/8$ refer to Algorithm \ref{alg-vanilla-variant} with the corresponding step-sizes. In particular, $\alpha_t = 1$ is equivalent to Algorithm 2. CEASE refers to Algorithm 4 with $\alpha=0.15 p/n$. }\label{fig-stepsize}
\end{figure}

The new algorithm also has some drawbacks. For instance, it implicitly assumes that the local sample size is large enough so that each local machine can solve its optimization problems reliably.
Consider the example of the distributed logistic regression with $N$ samples and $p$ variables in total. When $m$ is large so that $n = N/m < p$, the loss functions $\{ f_k \}_{k=1}^m$ on local machines are no longer strongly convex and thus do not have unique minima. Then Algorithm \ref{alg-vanilla-variant} needs to be modified because
\[
\btheta_{t,k}=\argmin_{\btheta}\left\{ f_{k}(\btheta) + g(\btheta) - \langle \nabla f_{k}(\btheta_{t}) - \nabla f(\btheta_{t}) ,\btheta\rangle \right\}
\]
is not uniquely defined. We need some pivoting rule to choose one optimal solution, such as the one with the minimum Euclidean norm. That complicates the algorithm and it is not clear how to establish theoretical guarantees then. In contrast, the quadratic proximity term in CEASE always make the objective function strongly convex. It ensures the uniqueness of $\{\btheta_{t,k} \}_{k=1}^m $ and facilitates computation.

\section{Distributed linear regression}\label{appendix-quadratic}
In distributed linear regression, recall that the $k^{th}$ machine defines a quadratic loss function
\[
\frac{1}{2n} \sum_{i\in\mathcal I_k} (y_{i} - \bx_{i} ^{\top} \btheta )^2 = \frac{1}{2}\btheta^{\top}\hat{\bSigma}_{k}\btheta-\hat{\bw}_{k}^{\top}\btheta + \frac{1}{2n} \sum_{i \in \mathcal{I}_k} y_i^2,
\]
where $\hat\bSigma_k = \frac{1}{n} \sum_{i\in\mathcal I_k} \bx_{i} \bx_{i}^{\top}$ and $\hat\bw_k = \frac{1}{n} \sum_{i\in\mathcal I_k} \bx_{i} y_{i}$. Let $f(\btheta)=\frac{1}{2}\btheta^{\top}\hat{\bSigma}\btheta-\hat{\bw}^{\top}\btheta$. Without loss of geneality we write $\bx_{i}=( 1,\bu_{i}^{\top} )^{\top} \in \R^{p}$.  

\begin{ass}\label{assump-subg-cov-samplesize}
	\begin{itemize}
		\item
		$\E \bu_i = 0$ and $\E (\bu_i \bu_i^{\top} ) = \bSigma \succ 0$. $\{ \bSigma^{-1/2 }\bu_{i} \}_{i=1}^{N}$ are i.i.d.  sub-Gaussian random vectors with bounded $\| \bSigma^{-1/2 } \bu_i \|_{\psi_2}$.
		\item The minimum eigenvalue $\lambda_{\min}(\bSigma)$ is bounded away from zero.
		
		\item
		$N / \Tr(\bSigma) \geq C >0$ and $n/\log m \geq c>0$ where $C$ and $c$ are constants.
	\end{itemize}
\end{ass}

For the least-squares, Algorithm 3 admits a close-form:
\begin{align*}
\btheta_{t+1} & =[\bI-(\hat{\bSigma}_{1}+\alpha\bI)^{-1}\hat{\bSigma}]\btheta_{t}+(\hat{\bSigma}_{1}+\alpha\bI)^{-1}\hat{\bw},
\end{align*}
and so does Algorithm 4:
\begin{align*}
\btheta_{t+1,k} & =[\bI-(\hat{\bSigma}_{k}+\alpha\bI)^{-1}\hat{\bSigma}]\btheta_{t}+(\hat{\bSigma}_{k}+\alpha\bI)^{-1}\hat{\bw},
\\
\btheta_{t+1} & =\left( \bI- \frac{1}{m}\sum_{k=1}^{m}(\hat{\bSigma}_{k}+\alpha\bI)^{-1} \hat{\bSigma}  \right) \btheta_{t}+ \frac{1}{m}\sum_{k=1}^{m}(\hat{\bSigma}_{k}+\alpha\bI)^{-1} \hat{\bw}.
\end{align*}
Intuitively, the averaging step in Algorithm 4 reduces variance and accelerates convergence. Below we study Algorithm 4 with the help of these analytical expressions. In the large sample regime, we achieve a contraction factor of $O(p/n)$ without any condition number; in the general regime, linear convergence is still guaranteed.

\begin{thm}\label{thm-dane-lm}
	Suppose Assumption \ref{assump-subg-cov-samplesize} holds and $n / p$ is bounded away from zero.  Then, there exist positive constants $C_1, C_2$ and $C_3$ such that when (i) $n \geq C_1 p$ and $\alpha\geq 0$ or (ii) $\alpha\geq C_1 \Tr(\bSigma)/n$, with probability tending to 1,
	\begin{align}
	\|  \btheta_{t} - \hat{\btheta}  \|_2 \leq
	2\sqrt{\kappa} \thickspace\eta^t \| \btheta_0- \hat{\btheta}  \|_2,\qquad \forall t\geq 0,
	\label{ineq-contraction-q}
	\end{align}
	where $\kappa=\lambda_{\max} (\bSigma) / \lambda_{\min}(\bSigma) $ and
	$\eta = 1 - \frac{ 1 - \min\{ 1/2, C_2 p/n \} }{1+ C_3\alpha}$.
	
\end{thm}



Theorem \ref{thm-dane-lm} reveals the following remarkable facts about Algorithm 4: No matter what relationship $n$ and $p$ have, proper regularization always guarantees linear convergence, and the rate exhibits a smooth transition as $p/n$ grows. Hence we can handle the distributed statistical estimation problem without assuming large enough $n$, overcoming the difficulty of other algorithms in literature \citep{ZDW13,BFL18,JLY18}.

If $n/p$ is large enough, the regularization is not necessary, but choosing $\alpha \asymp p/n$ does not hurt much. This is because we can control the contraction factor as:
\[
1 - \frac{1- C_2 p/n }{1+ C_3 \alpha } = \frac{C_3 \alpha + C_2 p/n}{1+ C_3\alpha} = O(p/n).
\]
When $n/p$ is not that large, most distributed statistical estimation procedures fail. By choosing $\alpha = \tilde{C} \Tr(\bSigma)/n$ for $\tilde{C}>C_1$ (see Condition (ii) of Theorem~\ref{thm-dane-lm}) we still have linear convergence with contraction factor at most
\[
1 - \frac{1- 1/2 }{1+ C_3 \alpha }  = 1 - \frac{1}{2+ 2C_3\tilde{C} \Tr(\bSigma)/n } < 1.
\]
In most situations of interest
we have $\Tr(\bSigma)\asymp p$ (even for pervasive factor models). Therefore we see that $\alpha \asymp p/n$ is a universal and adaptive choice of regularization over all the possible relation between $n$ and $p$.

Another benefit of the Algorithms is that the condition number $\kappa$ has only logarithmic effect on the iteration complexity, and the contraction factor in Theorem \ref{thm-dane-lm} does not depend on $\kappa$ at all. This is in stark contrast to the analysis under the same setting in \cite{SSZ14}, and helps relax the commonly used boundedness assumption on the  condition number in \cite{ZDW13}, \cite{BFL18}, \cite{JLY18}, among others. 
It is worth mentioning that \cite{WRX18} derive similar results for distributed linear regression when the local sample size $n$ is sufficiently large.

\section{Distributed logistic regression}\label{appendix-logistic}

In this section, we demonstrate the iterates of CEASE in distributed logistic regression. For illustration purposes, we use Newton's method on each local machine to solve the optimization problem. It is worth pointing out that Newton's method is not the only choice. First-order methods such gradient descent can also efficiently do the job.

Given samples $\{(\bx_i, y_i)\}_{i\in[N]}$, the loss on the $k$th machine is
\[
f_k(\btheta) = \frac{1}{n} \sum_{i\in\mathcal I_k} \bigg[\log(1+e^{\btheta^\top \bx_i})-y_i(\btheta^\top \bx_i)\bigg].
\]
Thus for Algorithm 4, we have for each iteration
\begin{align*}
\btheta_{t, k} &= \argmin_{\btheta} \bigg\{ f_k(\btheta) - \langle \nabla f_k(\btheta_t) - \nabla f(\btheta_t), \btheta\rangle + \frac{\alpha}{2} \|\btheta - \btheta_t\|_2^2\bigg\}\\
& = \argmin_{\btheta} \bigg\{\frac{1}{n}\sum_{i\in \mathcal I_k}\log(1+e^{\btheta^\top \bx_i}) - \langle \hat \bw_k +\bfsym {\eta}_{t, k}+\alpha\btheta_t, \btheta\rangle + \frac{\alpha}{2}\|\btheta\|_2^2\bigg\}.
\end{align*}
Here $\hat \bw_k = \frac{1}{m}\sum_{i\in \mathcal I_k}y_i\bx_i$, $\bfsym {\eta}_{t, k} = \nabla f_k(\btheta_t) - \nabla f(\btheta_t)$. Assume that we start with $\btheta_{t, k}^{(0)} = \btheta_t$. Then one Newton iteration of the above minimization problem is
$$
\btheta_{t, k}^{(j+1)} = \btheta_{t, k}^{(j)} - \bigg(\frac{1}{n}\bX_k^\top \bW_k^{(j)}\bX_k + \alpha \bI\bigg)^{-1}\bigg[\bX_k^\top(\bp_k^{(j)} - \bY_k)- \bfsym {\eta}_{t, k} + \alpha(\btheta_{t, k}^{(j)}-\btheta_t)\bigg].
$$
Here, $\bX_k\in \mathbb R^{n\times p}$ is a matrix with rows consisting of $\{\bx_i\}_{i\in \mathcal I_k}$, $\bY_k\in \mathbb R^{n}$ denotes the vector of $\{y_i\}_{i\in \mathcal I_k}$, $\bp_k^{(j)}\in \mathbb R^n$ is the vector of fitted probabilities with entry $i$ equal to $e^{(\btheta_{t, k}^{(j)})^\top \bx_i} / (1 + e^{(\btheta_{t, k}^{(j)})^\top \bx_i})$, and $\bW_k^{(j)}\in \mathbb R^{n\times n}$ is a diagonal matrix with entries $ \{ e^{(\btheta_{t, k}^{(j)})^\top \bx_i} / (1 + e^{(\btheta_{t, k}^{(j)})^\top \bx_i})^2\}_{i\in \mathcal I_k}$. Running one-step or multi-step Newton's iteration will result in some $\btheta_{t, k}^{(j)}$ sufficiently close to $\btheta_{t, k}$. After that, we simply average them across all the machines and obtain $\btheta_{t+1}$.

For Algorithm 2, in each iteration, we simply set $\alpha = 0$ in the above procedure. For Algorithm 3 and Algorithm 1, we simply omit the averaging step in Algorithm 4 and Algorithm 2 respectively.

\section{Numerical results on distributed $\ell_1$-regularized logistic regression}\label{appendix-l1_logistic_simulations}

In this section, we present the performance of our algorithms in the distributed $\ell_1$-regularized logistic regression problem where nonsmooth penalty is present. We fix the total sample size $N=5000$ and the dimensionality $p=1001$, and generate the i.i.d. data $\{ ( \bx_i, y_i ) \}_{i=1}^N $ as follows: $\bx_i = (1,\bu_i^{\top})^{\top}$ with $\bu_i \sim N( \mathbf{0}_{p-1} , \bI_{p-1} )$; $\P(y_i=1) = 1 - \P(y_i = 0) = \frac{e^{ \bx_i^{\top}\btheta^* }}{ 1 + e^{ \bx_i^{\top}\btheta^* } }$ where $\btheta^* = ( \mathbf{1}_{10}^{\top} , \mathbf{0}_{991}^{\top} )^{\top}  /\sqrt{2} \in \R^{p}$. We define the penalty function $g(\btheta) = \lambda \| \btheta \|_1$ with $\lambda = 0.5 \sqrt{ \frac{\log p}{N} }$, such that the regularized MLE over the whole dataset recovers the nonzeros of $\btheta^*$ accurately. Figure \ref{fig-L1} shows the performance of CEASE algorithms and ADMM, where ``large $n$", ``moderate $n$" and ``small $n$" refer to $(n,m)=(1000,5)$, $(500,10)$ and $(250,20)$, and ``zero initialization" and ``good initialization" refer to $\btheta_0 = \mathbf{0}$ and $\bar{\btheta}$, respectivley. Again, $\bar{\btheta}$ is the one-shot distributed estimator \citep{ZDW13}. All the results are average values of 100 independent runs.

Similar to the distributed logistic regression case, the CEASE algorithms with proper regularization (Algorithms 3 and 4) work well in general; without regularization, the CEASE algorithm fails to converge when the local sample size $n$ is small and the initialization in uninformative. For this nonsmooth problem, the CEASE algorithm with averaging (Algorithm 4) does not seem to have advantage over the single version (Algorithm 3). The ADMM converges quickly to a region near the minimizer but then proceeds quite slowly, which appears to be a common phenomenon in many distributed optimization problems \citep{BPC11}.

\begin{figure}[ht]
	\centering
	\subfigure{
		\includegraphics[width=1\linewidth]{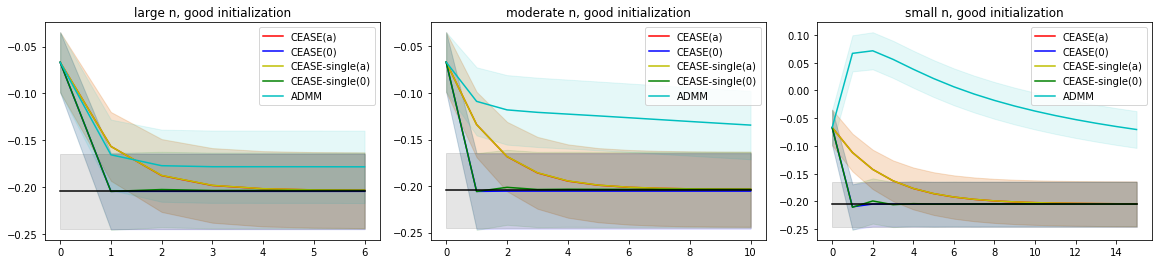}  
	}
	\\
	\subfigure{
		\includegraphics[width=1\linewidth]{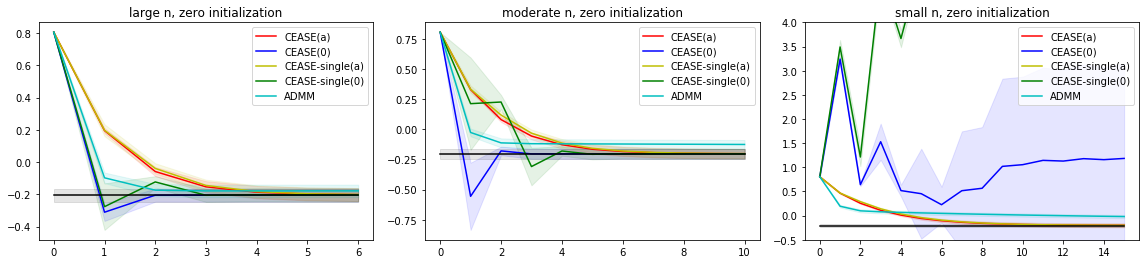}    
	}
	\caption{Nonsmooth minimization. The $x$-axis and $y$-axis are the number of iterations and $\log \| \btheta_t - \btheta^* \|_2$. The dashed lines show the error of the minimizer of the overall regularized loss function. The top and bottom panels use $\bar {\btheta}$ and $\mathbf{0}$ for initialization, respectively. CEASE(a) and CEASE(0) refer to Algorithm 4 with $\alpha=0.05 p/n$ and $0$; CEASE-single(a) and CEASE-single(0) refer to Algorithm 3 with $\alpha=0.05 p/n$ and $0$, respectively. In particular, CEASE-single(0) is equivalent to the CSL algorithm in \citep{JLY18}.}\label{fig-L1}
\end{figure}

\section{Proofs}\label{appendix-proofs}

\subsection{Proof of Theorem 3.1}
\begin{lem}\label{lem-thm-main}
	Let Assumptions 3.1 and 3.2 hold. Consider the iterates $\{ \btheta_t \}_{t=0}^{T}$ generated by Algorithm 4. Define
	\begin{align*}
		\gamma_t =
		\begin{cases}
			\frac{\delta}{\rho_0+\alpha} \cdot \min \{ 1 , \frac{\delta}{\rho+\alpha} (1+ \frac{M}{\rho_0+\alpha} \| \btheta_t - \hat\btheta\|_2 ) \} &, \text{ if Assumption 3.3 holds}\\
			\frac{\delta}{\rho_0 + \alpha} &, \text{ otherwise }
		\end{cases},\qquad 0\leq t\leq T-1.
	\end{align*}
	If $0<\| \btheta_t - \hat\btheta \|_2< R/2$, $\delta < \rho_0+\alpha$ and $\gamma_t^{2} < \rho / ( \rho+2\alpha )$, then
	\begin{align*}
		\frac{ \|\btheta_{t+1} - \hat\btheta \|_2 }{ \|\btheta_t - \hat\btheta \|_2 } \leq \frac{ \gamma_t \sqrt{\rho^2 + 2\alpha\rho} + \alpha }{\rho+\alpha} < 1.
	\end{align*}	
\end{lem}
Theorem 3.1 directly follows from Lemma \ref{lem-thm-main} and induction. Below we only prove Lemma \ref{lem-thm-main} with Assumption 3.3. The other part in Lemma \ref{lem-thm-main} without Assumption 3.3 can be derived by slightly modifying this proof.

\begin{proof}[\bf Proof of Lemma \ref{lem-thm-main} with Assumption 3.3]
	Let $\btheta_t^+ =  \mathrm{prox}_{\alpha^{-1} (f+g) } (\btheta_t) $. By the triangle inequality,
	\begin{align}
		\| \btheta_{t+1} - \hat\btheta \|_2
		\leq \| \btheta_{t+1} -\btheta_t^+\|_2  + \| \btheta_t^+ - \hat\btheta\|_2.
		\label{ineq-proof-thm-main-0}
	\end{align}
	
	We first invoke Theorem 3.2 to bound the first term $\| \btheta_{t+1} -\btheta_t^+\|_2$ in (\ref{ineq-proof-thm-main-0}). Define $\tilde g(\btheta) = g(\btheta) + (\alpha/2) \|\btheta - \btheta_t\|_2^2$ for $\btheta \in \R^p$. Then $\btheta_{t+1}$ is the first iterate of Algorithm 2 initialized at $\btheta_t$ for computing $\btheta^+_t = \argmin_{\btheta\in\R^p} \{ \frac{1}{m} \sum_{k=1}^{m} f_k(\btheta) + \tilde{g}(\btheta) \}$.
	
	From $\hat\btheta = \mathrm{prox}_{\alpha^{-1} (f+g)} (\hat\btheta)$ and Lemma \ref{lem-prox-nonexpansiveness} we obtain that
	$\|\btheta_t^+ - \hat\btheta\|_2 \leq \| \btheta_t - \hat\btheta\|_2$. Then the condition $\|\btheta_t -\hat\btheta \|_2<R/2$ leads to $B(\btheta_t^+ , R/2) \subseteq B(\hat\btheta,R)$. By Assumptions 3.1 and 3.2,
	\begin{itemize}
		\item in $B(\btheta^+_t,R/2)$, $\{ f_k+\tilde g \}_{k=1}^{m}$ are $(\rho_0+\alpha)$-strongly convex and $(f+\tilde g)$ is $( \rho+ \alpha)$-strongly convex;
		\item $\| \nabla^2 f_k(\btheta) - \nabla^2 f(\btheta) \|_2 \leq \delta$ holds for all $k\in[m]$ and $\btheta \in B(\btheta^+_t,R/2)$.
	\end{itemize}
	Furthermore, Lemma \ref{lem-prox-nonexpansiveness} also yields
	\begin{align}
		\| \btheta_t - \btheta_t^+ \|_2^2
		\leq \|\btheta_t - \hat \btheta \|_2^2 - \| \btheta_t^+ - \hat\btheta\|_2^2 =\|\btheta_t - \hat \btheta \|_2^2 \left( 1 - \| \btheta_t^+ - \hat\btheta\|_2^2/ \|\btheta_t - \hat \btheta \|_2^2 \right).
		\label{ineq-proof-thm-main-1}
	\end{align}
	Then $\| \btheta_t - \btheta_t^+ \|_2 \leq \| \btheta_t - \hat\btheta \|_2 < R/2$ and $\tilde{\btheta}_0 \in B(\btheta^+_t, R/2) $. Based on these conditions and $\alpha+\rho_0>\delta$, we use Theorem 3.2 to get
	\begin{align*}
		\| \btheta_{t+1} -\btheta_t^+\|_2  \leq \gamma_t \| \btheta_{t} -\btheta_t^+\|_2.
	\end{align*}
	
	From here, (\ref{ineq-proof-thm-main-0}) and (\ref{ineq-proof-thm-main-1}) we obtain that
	\begin{align*}
		\| \btheta_{t+1} - \hat\btheta \|_2
		&\leq \gamma_t \| \btheta_{t} -\btheta_t^+\|_2 + \| \btheta_t^+ - \hat\btheta\|_2 \notag \\
		&\leq \gamma_t \|\btheta_t - \hat \btheta \|_2 \left( 1 - \| \btheta_t^+ - \hat\btheta\|_2^2/ \|\btheta_t - \hat \btheta \|_2^2 \right)^{1/2} + \| \btheta_t^+ - \hat\btheta\|_2 \notag \\
		&= \|\btheta_t - \hat \btheta \|_2 \cdot h( \| \btheta_t^+ - \hat\btheta\|_2 / \|\btheta_t - \hat \btheta \|_2 ),
	\end{align*}
	where $h(x) = \gamma_t \sqrt{1-x^2}+x$, $\forall x\in[0,1]$. From $h'(x)=1-\gamma_t x/\sqrt{1-x^2}$ we see that $h' \geq 0$ on $[0,1/\sqrt{1+\gamma_t^{2} } ]$.
	
	On the one hand, Lemma \ref{lem-prox-nonexpansiveness} asserts that $\|\btheta^+ - \hat\btheta \|_2/\|\btheta_t-\hat\btheta\|_2 \leq \alpha/(\rho + \alpha)$. On the other hand, the assumption $\gamma_t^{2} < \rho/(\rho+2\alpha)$ forces
	\begin{align*}
		\frac{1}{\sqrt{1+\gamma_t^{2}}} > \frac{1}{\sqrt{1+ \rho/(\rho+2\alpha) }} = \frac{\sqrt{\rho/2+\alpha}}{\sqrt{\rho+\alpha }} \geq \frac{\rho/2+\alpha}{\rho+\alpha} \geq \alpha/(\rho + \alpha).
	\end{align*}
	The proof is completed by computation:
	\begin{align*}
		\frac{ \| \btheta_{t+1} - \hat\btheta \|_2 }{ \|\btheta_t - \hat \btheta \|_2 }
		&\leq h\left( \frac{ \| \btheta_t^+ - \hat\btheta\|_2 }{ \|\btheta_t - \hat \btheta \|_2 } \right) \leq h\left( \frac{\alpha}{\rho+\alpha} \right) = \gamma_t \left[ 1 - \left( \frac{\alpha}{\rho+\alpha} \right)^2 \right]^{1/2} + \frac{\alpha}{\rho+\alpha}\\
		&=\frac{\gamma_t \sqrt{\rho^2 + 2 \rho \alpha} + \alpha }{ \rho + \alpha }
		< \frac{ \sqrt{  [\rho/(\rho+2\alpha)] \cdot (\rho^2 + 2 \rho \alpha ) } + \alpha }{ \rho + \alpha } =1,
	\end{align*}
	where we used the assumption $\gamma_t^{2} < \rho/(\rho+2\alpha)$ again.
\end{proof}

\subsection{Proof of Corollary 3.1}
We claim that $( \frac{ \delta }{ \rho_0+\alpha } )^{2} \leq \frac{7}{9} \cdot  \frac{ \rho }{ \rho+2\alpha } $. Given this, Corollary 3.1 follows from Theorem 3.1 and
\[
\frac{ \frac{\delta}{\rho_0+\alpha} \sqrt{\rho^2 + 2\alpha\rho} + \alpha }{\rho+\alpha} \leq
\frac{  \sqrt{ \frac{7}{9} \cdot  \frac{ \rho }{ \rho+2\alpha }\cdot \rho (\rho+ 2\alpha) } + \alpha }{\rho+\alpha}
= 1 - \frac{ (1-\sqrt{7/9}) \rho }{\rho + \alpha } \leq 1 - \frac{ \rho/10 }{\rho + \alpha }.
\]

The claim trivially holds if $\delta=0$. When $\delta>0$, let us first assume $0<\delta \leq \rho$ and define $b=\alpha \rho / \delta^2$. Then $\alpha = b \delta^2 / \rho$, $b\geq 4$ and $\rho_0 \geq \max\{ \rho - \delta , 0 \}$ force
\begin{align*}
	\rho_0 + \alpha \geq \rho - \delta+ b \delta^2 / \rho =  ( \rho/\delta + b \delta/\rho -1 ) \delta \geq ( 2 \sqrt{ (\rho/\delta) \cdot (b \delta / \rho) } - 1 ) \delta = (2\sqrt{b}-1)\delta > \delta.
\end{align*}
and $ [ \delta / ( \rho_0+\alpha ) ]^{2} \leq [ \delta / ( \rho_0+\alpha ) ]^2  \leq 1/ h_1(\delta/\rho)$, where $h_1(x)=(bx+x^{-1}-1)^2$. On the other hand, $\rho / ( \rho+2\alpha ) = 1/( 1+2\alpha/\rho) = 1/ h_2(\delta/\rho) $, where $h_2(x)=1+2bx^2$.

We are going to show $h_2(x) \leq 7 h_1(x) /9$, $\forall x \in (0,1]$, which leads to the desired result under $0<\delta \leq \rho$. If $0<x\leq \sqrt{3}/2$, then $h_1(x) \geq (2\sqrt{b}-1)^2 \geq (2\sqrt{b}-\sqrt{b}/2)^2 \geq 9b/4$ and
\begin{align*}
	h_2(x) \leq 1 + 2 b \cdot (3/4) \leq (b/4) + (6b/4) = 7b/4 \leq 7h_1(x)/9.
\end{align*}
If $\sqrt{3}/2 < x \leq 1$, then $h_1(x) \geq b^2x^2 \geq 3b^2/4$, $h_2(x) \leq 1 + 2 b \leq (b/4)+2b = 9b/4$, and  $h_2(x)/h_1(x) =3/b \leq 3/4 \leq 7/9$.

Suppose now that $\delta > \rho$, and define $b=\alpha \rho / \delta^2$. Then
\begin{align*}
	&\left( \frac{ \delta }{ \rho_0+\alpha } \right)^{2} \leq \left( \frac{ \delta }{ \alpha } \right)^{2} = \left( \frac{ 1 }{ b \delta/\rho} \right)^{2}= \frac{ 1 }{ b^{2} (\delta/\rho)^{2} }, \\
	&\frac{\rho}{\rho+2\alpha} = \frac{1}{1+2\alpha/\rho} = \frac{1}{1+2b(\delta/\rho)^{2} }.
\end{align*}
From $b\geq 4$ and $\delta / \rho >1$ we get $( \frac{ \delta }{ \rho_0+\alpha } )^{2} \leq \frac{7}{9} \cdot  \frac{ \rho }{ \rho+2\alpha } $ from
\begin{align*}
	&b^{2} (\delta/\rho)^{2} - \frac{9}{7}[ 1+2b(\delta/\rho)^{2} ] = (\delta/\rho)^{2} b (b -18/7) - 9/7 \\
	&\geq 1\cdot 4 \cdot (4^1-18/7)-9/7 = 31/7 > 0.
\end{align*}

\subsection{Proof of Corollary 3.2 }
Throughout the proof we assume that $\delta/\rho$ is sufficiently small. The regularity conditions in Theorem 3.1 are easily verified as $\btheta_0 \in B(\hat\btheta,R/2)$ and $[ \delta / ( \rho_0+\alpha ) ]^{2} < \rho / ( \rho+2\alpha )$. Here we used the fact $\rho_0 \geq \rho - \delta$.

From $\rho_0 \geq \rho - \delta$ we get $\rho_0+\alpha \geq \rho_0\geq \rho/2$ and $\delta/(\rho_0+\alpha) \leq 2 \delta/\rho$. Also, $\sqrt{\rho^2 + 2 \alpha \rho } \leq \rho \sqrt{ 1 + 2C (\delta/\rho)^2 } \lesssim \rho$. We control the contraction factor in (3.2):
\begin{align*}
	\frac{ \frac{\delta}{\rho_0+\alpha} \sqrt{\rho^2 + 2\alpha\rho} + \alpha }{\rho+\alpha} \lesssim \frac{ (2 \delta/\rho)  \rho + C\delta^2/\rho }{\rho} =  \frac{2 \delta }{\rho} + \frac{C}{4} \left( \frac{2 \delta }{\rho} \right)^2
	\lesssim \frac{\delta }{\rho} .
\end{align*}

Recall that $\gamma_t =
\frac{\delta}{\rho_0+\alpha} \cdot \min \{ 1 , \frac{\delta}{\rho+\alpha} (1+ \frac{M}{\rho_0+\alpha} \| \btheta_t - \hat\btheta\|_2 ) \}$ in Theorem 3.1. When $\delta/\rho$ is small and $\| \btheta_t - \hat\btheta\|_2 \leq \rho/M$, we have $\rho_0 + \alpha \geq \rho/2$, $\frac{M}{\rho_0+\alpha} \| \btheta_t - \hat\btheta\|_2 \leq 2$, and $\gamma_t \leq (2\delta/\rho)^2$. This help bound the contraction factor in (3.2):
\begin{align*}
	\frac{ \gamma_t\cdot \sqrt{\rho^2 + 2\alpha\rho} + \alpha }{\rho+\alpha} \lesssim \frac{ (2 \delta/\rho)^{2} \rho + C\delta^2/\rho }{\rho}
	\lesssim \left( \frac{ \delta }{\rho} \right)^{2}.
\end{align*}

\subsection{Proof of Theorem 3.2}

Theorem 3.2 is a direct summary of the following two lemmas.

\begin{lem}[Contraction]\label{lem-vanilla-contraction}
	Let Assumptions 3.1 and 3.2 hold, with $\rho_0 > \delta \geq 0$. Then
	$\| \varphi_k(\btheta) - \hat\btheta \|_2 \leq (\delta / \rho_0) \| \btheta - \hat\btheta\|_2$, $\forall \btheta \in B(\hat\btheta , R)$, $\forall k \in [m]$.
\end{lem}

\begin{proof}
	
	Fix $\btheta \in B(\hat\btheta,R)$.  By the first order condition of $\varphi_k(\btheta)$, we have that
	\begin{align}
		\nabla f_k(\btheta) -\nabla f(\btheta)\in \partial \{ f_k [ \varphi_k(\btheta) ] + g [ \varphi_k(\btheta) ] \} .
		\label{eqn-vinilla-contraction-subg}
	\end{align}
	Using the fixed point property $\varphi_k(\hat \btheta) = \hat \btheta$, we have
	$\nabla f_k(\hat\btheta)-  \nabla f(\hat\btheta)  \in \partial [ f_k(\hat\btheta) + g(\hat\btheta) ]$.
	By the Taylor expansion and Assumption 3.2,
	\begin{align*}
		& \| [ \nabla f_{k}(\btheta) - \nabla f(\btheta) ] - [ \nabla f_{k}(\hat\btheta) - \nabla f(\hat\btheta) ] \|_2\notag\\
		= & \left\| \int_{0}^{1} \left(  \nabla^2 f_k [ (1-t) \hat\btheta + t \btheta ] - \nabla^2 f [ (1-t) \hat\btheta + t \btheta ]  \right) (\btheta - \hat\btheta) \mathrm{d} t \right\|_2 \notag\\
		\leq & \sup_{\bzeta \in B(\hat\btheta,R)} \| \nabla^2 f_k (\bzeta) - \nabla^2 f(\bzeta) \|_2 \cdot \| \btheta - \hat\btheta \|_2\\
		\leq & \delta \| \btheta - \hat\btheta \|_2 < \rho_0 R.
	\end{align*}
	From this, (\ref{eqn-vinilla-contraction-subg}) and Lemma \ref{lem-cvx-inverse}, we obtain that $ \| \varphi_k(\btheta) - \hat\btheta \|_2 \leq (\delta/\rho_0) \| \btheta - \hat\btheta \|_2$.
	
\end{proof}

\begin{lem}[Averaging]\label{lem-vanilla-averaging}
	Let Assumptions 3.1, 3.2 and 3.3 hold, with $\rho_0 > \delta \geq 0$. We have
	\begin{align*}
		\left\|  \frac{1}{m} \sum_{k=1}^{m} \varphi_k(\btheta) - \hat\btheta \right\|_2 \leq \frac{\delta^2}{\rho_0\rho} (1+M \rho_0^{-1} \| \btheta - \hat\btheta\|_2 ) \| \btheta - \hat\btheta\|_2,\qquad \forall \btheta \in B(\hat\btheta , R).
	\end{align*}
\end{lem}

\begin{proof}
	
	Define $L_k(\btheta) = f_k(\btheta) + g(\btheta)$ and $L(\btheta) = f(\btheta) + g (\btheta)$ for $\btheta\in\R^p$. Then $\hat\btheta = \argmin_{\bxi\in\R^p} L(\bxi)$ and $\varphi_k(\btheta) = \argmin_{\bxi \in \R^p} \{ L_k(\bxi) - \langle \nabla L_k (\btheta) - \nabla L(\btheta) , \bxi  \rangle \}$. By the optimality conditions,
	\begin{align*}
		\nabla L_k [ \varphi_k(\btheta) ] - \nabla L_k (\btheta) + \nabla L(\btheta) = \mathbf{0} = \nabla L(\hat\btheta).
	\end{align*}
	After subtracting $\nabla L_k(\hat\btheta)$ from both sides and rearranging terms, we get
	\begin{align*}
		\nabla L_k [ \varphi_k(\btheta) ] - \nabla L_k(\hat\btheta)
		= [\nabla L_k (\btheta) - \nabla L_k(\hat\btheta)] - [\nabla L(\btheta) - \nabla L(\hat\btheta)].
	\end{align*}
	Note that the average of the right hand side over $k\in[m]$ is $\mathbf{0}$.
	
	Define $\bH_k=\int_{0}^{1} \nabla^2 L_k [ (1-t) \hat{\btheta} + t \varphi_k(\btheta) ] \mathrm{d} t$ for $k\in[m]$ and $\hat\bH = \nabla^2 L( \hat{\btheta})$. Then
	\begin{align*}
		& \nabla L_k [ \varphi_k(\btheta) ] - \nabla L_k(\hat\btheta) = \bH_k ( \varphi_k(\btheta) - \hat\btheta ) = \hat\bH ( \varphi_k(\btheta) - \hat\btheta ) + (\bH_k - \hat\bH) ( \varphi_k(\btheta) - \hat\btheta ), \\
		&\mathbf{0} = \frac{1}{m} \sum_{k=1}^{m} \left(  \nabla L_k [ \varphi_k(\btheta) ] - \nabla L_k(\hat\btheta) \right) = \hat\bH [ \bar\varphi( \btheta ) - \hat\btheta ] + \frac{1}{m} \sum_{k=1}^{m} (\bH_k - \hat\bH) ( \varphi_k(\btheta) - \hat\btheta ),
	\end{align*}
	where we let $\bar{\varphi}(\btheta) = \frac{1}{m} \sum_{k=1}^{m} \varphi_k(\btheta)$. As a result,
	\begin{align*}
		\|  \bar\varphi( \btheta ) - \hat\btheta \|_2
		&= \left\| \frac{1}{m} \sum_{k=1}^{m} \hat\bH^{-1} (\bH_k - \hat\bH) ( \varphi_k(\btheta) - \hat\btheta ) \right\|_2 \\
		&\leq \| \hat\bH^{-1} \|_{2} \max_{ k\in[m] } \| \bH_k - \hat\bH\|_2 \cdot \max_{ k\in[m] } \| \varphi_k(\btheta) - \hat\btheta \|_2.
	\end{align*}
	Lemma \ref{lem-vanilla-contraction} forces $\max_{ k\in[m] } \| \varphi_k(\btheta) - \hat\btheta \|_2 \leq (\delta / \rho_0 ) \| \btheta - \hat\btheta \|_2$, and Assumption 3.1 yields $\hat\bH \succeq \rho \bI$ and $\| \hat\bH \|_2 \leq 1/ \rho$. Furthermore, we use Assumptions 3.2 and 3.3 to get
	\begin{align*}
		\| \bH_k - \hat\bH \|_2
		&\leq \left\| \int_{0}^{1} \left(
		\nabla^2 L_k [ (1-t) \hat{\btheta} + t \varphi_k(\btheta) ] - \nabla^2 L [ (1-t) \hat{\btheta} + t \varphi_k(\btheta) ]
		\right) \mathrm{d} t \right\|_2 \\
		&+ \left\| \int_{0}^{1} \left(
		\nabla^2 L [ (1-t) \hat{\btheta} + t \varphi_k(\btheta) ] - \nabla^2 L(\hat\btheta)
		\right) \mathrm{d} t \right\|_2 \\
		&\leq \delta + M \| \varphi_k(\btheta)  - \hat\btheta \|_2
		\leq \delta + M ( \delta /\rho_0) \| \btheta - \hat\btheta\|_2.
	\end{align*}
	The proof is finished by combining all the estimates above.
\end{proof}

\subsection{Proof of Theorem 4.1}\label{section-pf-glm}

The proof is implied by combining proof of Corollary 3.2 with the results of the following two lemmas, the first of which is a direct counterpart of Theorem 3.1 in the stochastic setting, given an additional condition on similarity between local Hessians. The second lemma below specifies the order of Hessian difference in the generalized linear model, hence providing a contraction rate and guiding the choice of $\alpha$.

\begin{lem}\label{lem-main-general}
	Let Assumption 4.2 hold. Denote
	$$
	\hat\delta:=2\sup_{k\in[m]}\sup_{\btheta\in B(\hat\btheta,R)}\| \nabla^2 f_k(\btheta) - \nabla^2 F (\btheta) \|_2.
	$$
	Consider the iterates $\{ \btheta_t \}_{t=0}^{T}$ generated by Algorithm 3 or Algorithm 4. Suppose that $\btheta_0 \in B( \hat\btheta , R/2 )$ and $[\hat \delta / ( \rho_0+\alpha ) ]^{2} < \rho / ( \rho+2\alpha )$.
	\begin{itemize}
		\item For both Algorithms 3 and 4, we have
		\begin{align}
			\| \btheta_{t+1} - \hat\btheta \|_2 \leq \| \btheta_{t} - \hat\btheta\|_2 \cdot  \frac{ \frac{\hat\delta}{\rho_0+\alpha}\cdot \sqrt{\rho^2 + 2\alpha\rho} + \alpha }{\rho+\alpha}, \qquad 0\leq t \leq T-1;
			\label{ineq-thm-main-1-appendix}
		\end{align}
		\item If in addition, Assumption 4.3 also holds, then for Algorithm 4 we have
		\begin{align}
			\| \btheta_{t+1} - \hat\btheta \|_2 \leq \| \btheta_t - \hat\btheta\|_2 \cdot  \frac{ \gamma_t \sqrt{\rho^2 + 2\alpha\rho} + \alpha }{\rho+\alpha},\qquad 0\leq t\leq T-1,
			\label{ineq-thm-main-2-appendix}
		\end{align}
		where we define $\gamma_t =
		\frac{\hat\delta}{\rho_0+\alpha} \cdot \min \{ 1 , \frac{\hat\delta}{\rho+\alpha} (1+ \frac{M}{\rho_0+\alpha} \| \btheta_t - \hat\btheta\|_2 ) \}$;
		\item Both multiplicative factors in (\ref{ineq-thm-main-1-appendix}) and (\ref{ineq-thm-main-2-appendix}) are strictly less than 1.
	\end{itemize}
\end{lem}

\begin{proof}[\bf Proof of Lemma \ref{lem-main-general}]
	We first assume that $\rho_0>\hat\delta$ and analyze the vanilla DANE algorithm under the new assumptions. Let Assumption 4.2 hold, and $\{ \btheta_t \}_{t=0}^{\infty}$ be the iterates with $\btheta_0 \in B(\hat\btheta,R)$. For any $\btheta \in B(\hat\btheta,R)$, $ \| \nabla^2 f (\btheta) - \nabla^2 F(\btheta) \|_2 \leq \hat\delta/2$ and thus $\| \nabla^2 f_k (\btheta) - \nabla^2 f(\btheta) \|_2 \leq \hat\delta$ for $k\in[m]$. Hence it implies Assumption 3.2 with $\delta=\hat\delta$, and Lemma \ref{lem-vanilla-contraction} continues to hold. We can also get the result in Lemma \ref{lem-vanilla-averaging} under Assumption 4.3, by replacing $\hat\bH = \nabla^2 f(\hat\btheta)$ in the proof of Lemma \ref{lem-vanilla-averaging} by $\nabla^2 F(\hat\btheta)+\nabla^2 g(\btheta)$. Then we drop the assumption $\rho_0 >\hat \delta$ can reproduce the results in Theorem 3.1 under the new setting, by following its original proof.
\end{proof}

\begin{lem}\label{lem-glm-delta}
	Under Assumption 4.1, for an arbitrarily small positive constant $c$, there exist universal positive constants $C_1, C_2$ and $C_3$ depending only on $c$ such that as long as $n\geq cp$, with probability at least $1-2e^{-C_2 n}-Ne^{-C_3p}$,
	$$
	\sup_{k\in[m]}\sup_{\btheta\in B(\hat\btheta,R)}\| \nabla^2 f_k(\btheta) - \nabla^2 F (\btheta) \|_2\leq C_1\ltwonorm{\bSigma}\sqrt{\frac{p\max\{1,\log(Np^{1/2}\ltwonorm{\bSigma}R)\}}{n}}.
	$$
\end{lem}
\begin{proof}[\bf Proof of Lemma \ref{lem-glm-delta}]
	Let $\tilde{\bx}_i = (\bSigma^*)^{-1/2} \bx_i$, $\tilde\btheta=(\bSigma^*)^{1/2}\btheta$, and define a new loss function
	$
	\tilde l(\tilde\btheta,\tilde \bx_{i})=b(\tilde\bx_{i}^\top\tilde\btheta)-y_{i}(\tilde\bx_{i}^\top\tilde\btheta).
	$
	Let $\hat R_k(\tilde\btheta)=\frac{1}{n}\sum_{i\in\mathcal I_k}\tilde l(\tilde\btheta,\tilde\bx_{i})$ for $k\in[m]$. Then we have that
	$
	\nabla^2\hat R_k(\tilde\btheta)=\frac{1}{n}\sum_{i\in\mathcal I_k} b''(\tilde \btheta^\top\tilde \bx_{i})\tilde \bx_{i}\tilde \bx_{i}^\top
	$
	and $\nabla^2 f_k(\btheta)=(\bSigma^*)^{1/2}\nabla^2\hat R_k(\tilde \btheta)(\bSigma^*)^{1/2}$. Similarly we have $\nabla^2 F(\btheta)=(\bSigma^*)^{1/2}\E \nabla^2\hat R_k(\tilde \btheta)(\bSigma^*)^{1/2}$.
	
	Therefore
	\begin{equation}\label{eq-prf-glm-1}
		\max_{k\in[m]}\max_{\btheta\in{B(\hat\btheta, R)}}\| \nabla^2 f_k(\btheta) - \nabla^2 F (\btheta) \|_2 \leq
		\ltwonorm{\bSigma^*} \max_{k\in[m]}\max_{\tilde \btheta\in B((\bSigma^*)^{1/2}\hat\btheta,\tilde R)}\| \nabla^2 \hat R_k(\tilde \btheta) - \nabla^2 \E \hat R_k(\tilde \btheta) \|_2
	\end{equation}
	and we only need to control the quantity on the right hand side. Here $\tilde R=\ltwonorm{\bSigma^*}^{1/2}R$.
	
	Define
	$
	\bDelta_0=\max_{k\in[m]}\max_{\tilde \btheta\in B((\bSigma^*)^{1/2}\hat\btheta,\tilde R)}\| \nabla^2 \hat R_k(\tilde \btheta) - \nabla^2 \E \hat R_k(\tilde \btheta) \|_2
	$
	and
	$$
	\phi_k({\tilde \btheta})=\ltwonorm{\frac{1}{n}\sum_{i\in\mathcal I_k} b''({\tilde\bx}_{i}^\top{\tilde \btheta}){\tilde\bx}_{i}{\tilde\bx}_{i}^\top-\E b''({\tilde\bX}^\top{\tilde \btheta}){\tilde\bX}{\tilde\bX}^\top}.
	$$
	Here $\tilde\bX$ shares the distribution with $\tilde \bx_{i}$. Firstly we bound $\phi_k({\tilde \btheta})$ for any fixed ${\tilde \btheta}\in B({\tilde \btheta}^*, 2\tilde R)$, where ${\tilde \btheta}^*:=(\bSigma^*)^{1/2}\hat\btheta$. For any $k\in[m]$, under Assumption 4.1, there exist constants $c_1,c_2$ such that for any $\epsilon\geq 0$
	\begin{equation}\label{eq-phik}
		\P(\phi_k({\tilde \btheta})\geq \epsilon)\leq 2e^{c_1p-c_2\min{\{\epsilon,\epsilon^2\}}n}.
	\end{equation}
	To see this, notice that $\phi_k({\tilde \btheta})=\sup_{\bu\in \S^{p-1}} g_k(\bu)$, where
	$
	g_k(\bu)=\bu^\top \{\frac{1}{n}\sum_{i\in\mathcal I_k} b''({\tilde\bx}_{i}^\top{\tilde \btheta}){\tilde\bx}_{i}{\tilde\bx}_{i}^\top-\E b''({\tilde\bX}^\top{\tilde \btheta}){\tilde\bX}{\tilde\bX}^\top\}\bu.
	$
	Let $\mathcal N$ be a $\frac{1}{4}$-covering of $\S^{p-1}$, and $|\mathcal N|\leq 9^p$. Denote $\hat\bu=\arg\max_{\bu} g_k(\bu)$. Find $\tilde\bu\in\mathcal N$ such that $\ltwonorm{\hat\bu-\tilde\bu}\leq \frac{1}{4}$. Then
	$$
	|g_k(\tilde\bu)-g_k(\hat\bu)|=|(\tilde\bu+\hat\bu)^\top\left\{\frac{1}{n}\sum_{i\in\mathcal I_k} b''({\tilde\bx}_{i}^\top{\tilde \btheta}){\tilde\bx}_{i}{\tilde\bx}_{i}^\top-\E b''({\tilde\bX}^\top{\tilde \btheta}){\tilde\bX}{\tilde\bX}^\top\right\}(\tilde\bu-\hat\bu)|\leq \frac{1}{2}g_k(\hat\bu),
	$$
	and thus
	$$
	\sup_{\bu\in \S^{p-1}} g_k(\bu)\leq 2\sup_{\bu\in\mathcal N}g_k(\bu).
	$$
	On the other hand from Bernstein's inequality we see that there exists a constant $c_2$ such that for any $\bu\in \mathcal N$, $\epsilon\geq 0$, $\P(g_k(\bu)\geq \frac{\epsilon}{2})\leq 2e^{-c_2\min\{\epsilon,\epsilon^2\}n}$. Therefore
	$$
	\P(\phi_k({\tilde \btheta})\geq \epsilon)=\P(\sup_{\bu\in \S^{p-1}}g_k(\bu)\geq \epsilon)\leq \P(\sup_{\bu\in\mathcal N}g_k(\bu)\geq \frac{\epsilon}{2})\leq |\mathcal N| \cdot 2e^{-c_2\min\{\epsilon,\epsilon^2\}n}\leq 2e^{c_1p-c_2\min{\{\epsilon,\epsilon^2\}}n}
	$$
	where $c_1=\log 9$.
	
	Now for $t\geq 1$, define the event
	$
	E_t\triangleq \left\{ \max_{i=1}^{N}\ltwonorm{{\tilde\bx}_{i}}^3< (8t)^{3/2}\E \ltwonorm{{\tilde\bX}}^3\right\}.
	$
	Then by Theorem 2.1 in \cite{HKZ12}, $\P(E_t^c)= \P(\max_{i=1}^{N}\ltwonorm{{\tilde\bx}_{i}}^2\geq 8t[\E \ltwonorm{{\tilde\bX}}^3]^{2/3})\leq \P(\max_{i=1}^{N}\ltwonorm{{\tilde\bx}_{i}}^2\geq 8t\E \ltwonorm{{\tilde\bX}}^2)\leq N\P(\ltwonorm{{\tilde\bX}}^2\geq 8t\E \ltwonorm{{\tilde\bX}}^2)\leq Ne^{-tp}$. Under the event $E_t$, for ${\tilde \btheta}_1,{\tilde \btheta}_2\in B({\tilde \btheta}^*, 2\tilde R)$, we have
	\begin{align*}
		|\phi_k({\tilde \btheta}_1)-\phi_k({\tilde \btheta}_2)|&\leq \ltwonorm{\frac{1}{n} \sum_{i\in\mathcal I_k} [b''({\tilde\bx}_{i}^\top{\tilde \btheta}_1)-b''({\tilde\bx}_{i}^\top{\tilde \btheta}_2)]{\tilde\bx}_{i}{\tilde\bx}_{i}^\top}+\ltwonorm{\E b''({\tilde\bX}^\top{\tilde \btheta}_1){\tilde\bX}{\tilde\bX}^\top-\E b''({\tilde\bX}^T{\tilde \btheta}_2){\tilde\bX}{\tilde\bX}^\top}\\
		&\leq B_3\ltwonorm{{\tilde \btheta}_1-{\tilde \btheta}_2}\cdot(\E\ltwonorm{{\tilde\bX}}^3+\frac{1}{n}\sum_{i=1}^n\ltwonorm{{\tilde\bx}_{ki}}^3)\\
		& \leq (9t)^{3/2}B_3\cdot \E\ltwonorm{\bU}^3\cdot \ltwonorm{{\tilde \btheta}_1-{\tilde \btheta}_2}\\
		&\leq c_3(pt)^{3/2}\ltwonorm{{\tilde \btheta}_1-{\tilde \btheta}_2}
	\end{align*}
	for some constant $c_3$ depending only on $B_3$.
	
	Now let $\mathcal N_\delta$ be a $\delta$-covering of $B({\tilde \btheta}^*, 2\tilde R)$, where $\delta=\frac{\epsilon}{c_3(tp)^{3/2}}$. We can also assume that $|{\mathcal N}_\delta|\leq (\frac{6\tilde R}{\delta})^p$. Therefore for any $k\in[m]$,
	\begin{align*}
		&\P\left(E_t\cap\left\{\sup_{{\tilde \btheta}\in B({\tilde \btheta}^*, 2\tilde R)} \phi_k({\tilde \btheta})\geq 2\epsilon\right\}\right)\leq \P\left( E_t\cap\left\{\sup_{{\tilde \btheta}\in {\mathcal N}_\delta} \phi_k({\tilde \btheta})\geq \epsilon\right\}\right)\\
		\leq \thickspace& |{\mathcal N}_\delta| \cdot 2e^{c_1 p-c_2 \min\{\epsilon,\epsilon^2\}n}
		=2e^{c_4 p+c_3 p\log\frac{(tp)^{3/2}\tilde R}{\epsilon}-c_2\min\{\epsilon,\epsilon^2\}n}.
	\end{align*}
	Thus
	\begin{align*}
		&\P(\bDelta_0\geq 2\epsilon)\leq \P\left( \cup_{k\in[m]}\left\{ \sup_{{\tilde \btheta}\in B({\tilde \btheta}^*, {\tilde R})} \phi_k({\tilde \btheta})\geq 2\epsilon\right\}\right)\\
		\leq &\P(E_t^c)+\sum_{k=1}^m \P\left(E_t\cap\left\{\sup_{{\tilde \btheta}\in B({\tilde \btheta}^*, {\tilde R})} \phi_k({\tilde \btheta})\geq 2\epsilon\right\}\right)\\
		\leq & Ne^{-tp}+2me^{c_4 p+c_3 p\log\frac{(tp)^{3/2}{\tilde R}}{\epsilon}-c_2\min\{\epsilon,\epsilon^2\}n}.
	\end{align*}
	It is easily seen that the last expression is no more than $Ne^{-tp}+2e^{-C_2 p}$ when
	$$
	\begin{cases}
	\min\{\epsilon,\epsilon^2\}\geq C_1' \frac{\max\{p,p\log (t^{3/2}R), \log m\}}{n},\\
	\epsilon\geq 1\text{ or }\frac{\epsilon^2}{\log \frac{1}{\epsilon}}\geq C_2'\cdot \frac{p}{n},
	\end{cases}
	$$
	which is satisfied if $t$ is chosen as a suitable constant depending on $c$, and that
	\begin{equation}\label{eq-concentration-epsilon}
		\epsilon=C\sqrt{\frac{\log m+p\max\{1,\log(np^{1/2}\tilde R)\}}{n}}.
	\end{equation}
	Here $C_i'$ and $C$ is a constant depending only on $c$.
	
	Finally, note that
	$$\bSigma^*=\cov(\bx_i)=\begin{pmatrix}1&\\ &\bSigma\end{pmatrix}$$
	and that $\ltwonorm{\bSigma}\geq A_1$ for a universal $A_1>0$, we have $\ltwonorm{\bSigma^*}\leq \max\{1, 1/A_1\}\ltwonorm{\bSigma}$. Thus combining (\ref{eq-concentration-epsilon}) and (\ref{eq-prf-glm-1}) completes the proof.
\end{proof}


\subsection{Proof of Theorem 4.2}

Theorem 4.2 is a special of Theorem 4.1 by taking $\alpha=0$. 

\subsection{Proof of Theorem \ref{thm-dane-lm}}\label{section-pf-lm}

We first present three lemmas, based on which we build the proof of the main theorem.

\begin{lem}\label{lem-dane-quadratic-1}
	Suppose that $\hat\bSigma$ is positive-definite, i.e. $\lambda_{\min}(\hat\bSigma) >0$. Define $\bepsilon_{t}=\hat \bSigma^{1/2}(\btheta_{t}-\hat \bSigma^{-1}\hat \bw)$ for $t\geq 0$ and
	\[
	\bDelta_k=( \hat\bSigma + \alpha \bI )^{-1/2}(
	\hat{\bSigma}_k - \hat\bSigma
	) ( \hat\bSigma + \alpha \bI )^{-1/2},\qquad \forall k\in[m].
	\]
	If $\alpha\geq 0$ is appropriately chosen such that $\Delta = \max\limits _{k\in[m]}\| \bDelta_k \| _{2} \leq 1/2$. Then
	\[
	\| \bepsilon_{t+1} \|_2 \leq \frac{ 2\Delta^2  + \alpha / \lambda_{\min}(\hat\bSigma) }{ 1 + \alpha / \lambda_{\min}(\hat\bSigma) } \| \bepsilon_t \|_2,
	\qquad \forall t\geq 0,
	\]
	which guarantees linear convergence of $\left\{ \bepsilon_{t}\right\} _{t=0}^{\infty}$.
\end{lem}

\begin{proof}[\bf Proof of Lemma \ref{lem-dane-quadratic-1}]
	Define $\bepsilon_{t,k}=\hat \bSigma^{1/2}(\btheta_{t,k}-\hat \bSigma^{-1}\hat \bw)$. Then
	\begin{align*}
		&\bepsilon_{t+1,k}=\hat\bSigma^{1/2} ( \btheta_{t+1,k} - \hat\bSigma^{-1}\hat{\bw} ) \\
		& = \hat\bSigma^{1/2}[\bI-(\hat{\bSigma}_{k}+\alpha\bI)^{-1}\hat{\bSigma}]\btheta_{t}+
		\hat\bSigma^{1/2}(\hat{\bSigma}_{k}+\alpha\bI)^{-1}\hat{\bw}
		-\hat\bSigma^{-1/2}\hat{\bw}\\
		&=[ \bI- \hat \bSigma^{1/2}(\hat{\bSigma}_{k}+\alpha\bI)^{-1}
		\hat \bSigma^{1/2} ] \bepsilon_{t}.
	\end{align*}
	Define $\widetilde{\bSigma}_{k}^{(1)}=
	\hat \bSigma^{-1/2}\hat{\bSigma}_{k}\hat \bSigma^{-1/2}$
	for $k\in[m]$. The fact
	\begin{align*}
		\hat{\bSigma}_{k}+\alpha\bI=
		\hat{\bSigma}^{1/2} (
		\hat{\bSigma}^{-1/2} \hat{\bSigma}_k \hat{\bSigma}^{-1/2} + \alpha \hat{\bSigma}^{-1}
		)\hat{\bSigma}^{1/2}
		=\hat{\bSigma}^{1/2} (
		\widetilde{\bSigma}_k^{(1)} + \alpha \hat{\bSigma}^{-1}
		)\hat{\bSigma}^{1/2}
	\end{align*}
	gives $\hat \bSigma^{1/2}(\hat{\bSigma}_{k}+\alpha\bI)^{-1}
	\hat \bSigma^{1/2}
	= ( \widetilde{\bSigma}_k^{(1)} + \alpha \hat \bSigma^{-1} )^{-1}$ and $
	\bepsilon_{t+1,k}=
	[ \bI-  ( \widetilde{\bSigma}_k^{(1)} + \alpha \hat \bSigma^{-1} )^{-1} ] \bepsilon_{t}
	$.
	
	Define $\hat\bD=(\bI+\alpha \hat \bSigma^{-1})^{-1}$ and $\widetilde{{\bSigma}}_{k}=\hat\bD^{1/2}\widetilde{\bSigma}_{k}^{(1)} \hat\bD^{1/2}$.
	From
	\begin{align*}
		(\widetilde{\bSigma}_{k}^{(1)}+\alpha \hat\bSigma^{-1})^{-1}= & [ \hat\bD^{-1}+(\widetilde{\bSigma}_{k}^{(1)}-\bI)]^{-1}=\hat\bD^{1/2}[\bI+(\widetilde{\bSigma}_{k}-\hat\bD)]^{-1}\hat\bD^{1/2}
	\end{align*}
	and
	\begin{align*}
		&\widetilde{\bSigma}_k - \hat\bD = \hat\bD^{1/2} ( \widetilde{\bSigma}^{(1)}_k - \bI ) \hat\bD^{1/2}
		=( \bI + \alpha \hat\bSigma^{-1} )^{-1/2} \hat\bSigma^{-1/2}
		(
		\hat{\bSigma}_k - \hat\bSigma
		) \hat\bSigma^{-1/2} ( \bI + \alpha \hat\bSigma^{-1} )^{-1/2}\\
		&=( \hat\bSigma + \alpha \bI )^{-1/2}(
		\hat{\bSigma}_k - \hat\bSigma
		) ( \hat\bSigma + \alpha \bI )^{-1/2} = \bDelta_k,
	\end{align*}
	we get
	\begin{align*}
		\bepsilon_{t+1,k}=
		( \bI- \hat\bD^{1/2} \bC_k \hat\bD^{1/2} ) \bepsilon_{t}
		\quad\text{and} \quad
		\bepsilon_{t+1}=
		( \bI- \hat\bD^{1/2} \bC \hat\bD^{1/2} ) \bepsilon_{t},
	\end{align*}
	where $\bC_k = (
	\bI + \bDelta_k)^{-1}$ and $\bC = \frac{1}{m} \sum_{k=1}^{m} \bC_k$.
	Let $\bR_{k}=\bC_k-( \bI-\bDelta_k )$ and $\bR = \frac{1}{m} \sum_{k=1}^{m} \bR_k = \bC- \bI$. We have
	$\| \bI- \hat\bD^{1/2} \bC \hat\bD^{1/2} \| _{2}
	= \| \bI- \hat \bD^{1/2} ( \bI+ \bR ) \hat\bD^{1/2}\|_2$. Below we control the right-hand side.
	
	By
	$\Delta = \max\limits _{k\in[m]}\| \bDelta_k \| _{2} \leq 1/2$ and Lemma \ref{lem-neumann}, we obtain that $\| \bC_k \|_2 \leq\frac{{1}}{1-\Delta}\leq2$
	and $\| \bR_{k}\| _{2}\leq 2 \Delta^2$.
	Consequently, $\| \bC\| _{2}\leq 2$ and $\| \bR \|_2 \leq 2 \Delta^2 \leq 1/2$.
	Then we obtain that
	\begin{align*}
		& (1-2\Delta^2) \bI \preceq \bI+ \bR \preceq ( 1 + 2 \Delta^2) \bI ,\\
		&(1-2\Delta^2)\hat\bD \preceq \hat\bD^{1/2} ( \bI+ \bR ) \hat\bD^{1/2} \preceq ( 1 + 2 \Delta^2)\hat\bD,\\
		& \bI - (1+2\Delta^2)\hat\bD \preceq \bI- \hat \bD^{1/2} ( \bI+ \bR ) \hat\bD^{1/2} \preceq \bI - (1-2\Delta^2)\hat\bD.
	\end{align*}
	Consequently,
	\begin{align*}
		& \| \bI- \hat \bD^{1/2} ( \bI+ \bR ) \hat\bD^{1/2} \|_2 \leq
		\max\left\{
		\| \bI - (1-2\Delta^2)\hat\bD \|_2,\| \bI - (1+2\Delta^2)\hat\bD \|_2
		\right\}.
	\end{align*}
	Let $\{ \hat{\lambda}_j \}_{j=1}^p$ be the eigenvalues of $\hat{\bSigma}$ sorted in descending order. Since $\hat\bD$ has eigenvalues $\{ (1+\alpha/\hat{\lambda}_{j} )^{-1}
	\} _{j=1}^{p}\subseteq(0,1]$, the eigenvalues of $\bI - (1\pm 2\Delta^2)\hat\bD$ are
	$\left\{
	1 - \frac{1\pm 2\Delta^2 }{1+\alpha/\hat\lambda_j}
	\right\}_{j=1}^p$. Then
	\[
	\| \bI - (1 \pm 2\Delta^2)\hat\bD \|_2 =
	\max\left\{
	\left| 1 - \frac{1\pm 2\Delta^2}{1+ \alpha / \hat{\lambda}_1 } \right|,
	\left| 1 - \frac{1\pm 2\Delta^2}{1+ \alpha / \hat{\lambda}_p } \right|
	\right\}.
	\]
	By elementary calculation and the fact $2\Delta^2 \leq 1/2 < 1$ we get
	\begin{align*}
		&\left| 1 - \frac{1+ 2\Delta^2}{1+ \alpha / \hat{\lambda}_1 } \right|
		=\frac{ \left| \alpha / \hat\lambda_1 - 2\Delta^2 \right| }{1+ \alpha / \hat{\lambda}_1 }
		\leq \frac{ \max\{ \alpha / \hat\lambda_1 , 2\Delta^2 \} }{1+ \alpha / \hat{\lambda}_1 }
		\leq \max\left\{
		\frac{\alpha / \hat\lambda_p }{ 1 + \alpha / \hat\lambda_p },~
		2\Delta^2
		\right\}
		,\\
		&\left| 1 - \frac{1 + 2\Delta^2}{1+ \alpha / \hat{\lambda}_p } \right|
		=\frac{ \left| \alpha / \hat\lambda_p - 2\Delta^2 \right| }{1+ \alpha / \hat{\lambda}_p }
		\leq \frac{ 2\Delta^2  + \alpha / \hat\lambda_p }{1+ \alpha / \hat{\lambda}_p },\\
		&0 \leq 1 - \frac{1 - 2\Delta^2}{1+ \alpha / \hat{\lambda}_1 }  \leq 1 - \frac{1 - 2\Delta^2}{1+ \alpha / \hat{\lambda}_p }
		= \frac{ 2\Delta^2  + \alpha / \hat\lambda_p  }{1+ \alpha / \hat{\lambda}_p }.
	\end{align*}
	Therefore,
	\begin{align*}
		\| \bI- \hat \bD^{1/2} \bC \hat\bD^{1/2} \|_2 \leq
		\max\left\{
		\frac{ 2\Delta^2  + \alpha / \hat\lambda_p }{ 1 + \alpha / \hat\lambda_p },~
		2\Delta^2
		\right\}
		=\frac{ 2\Delta^2  + \alpha / \hat\lambda_p }{ 1 + \alpha / \hat\lambda_p }
		<1.
	\end{align*}
\end{proof}

\begin{lem}\label{lem-dane-quadratic-regularity}
	Suppose Assumption \ref{assump-subg-cov-samplesize} hold. Then there exists a constant $C$ determined by $\| \bu_{i} \|_{\psi_2}$, such that $\mathbb{P} ( \Delta \leq 1/2 ) \geq 1 - 2me^{-n/C}$ holds under either of the two conditions: (i) $n \geq C p$ and $\alpha \geq 0$; (ii) $\alpha \geq C \Tr(\bSigma) / n$.
\end{lem}

\begin{proof}[\bf Proof of Lemma \ref{lem-dane-quadratic-regularity}]
	Define $\hat\bS_{k} = \frac{1}{n} \sum_{i\in\mathcal I_k} \bu_{i} \bu_{i}^{\top}$ and $\bar{\bu}_{(k)}=\frac{1}{n} \sum_{\in\mathcal I_k} \bu_{i}$ for $k\in[m]$. Then we have $\hat\bSigma_{k} = \frac{1}{n} \sum_{i\in\mathcal I_k} \bx_{i} \bx_{i}^{\top} = \begin{pmatrix}
	1 & \bar{\bu}_{(k)}^{\top} \\
	\bar{\bu}_{(k)} & \hat\bS_k
	\end{pmatrix}$. Let $\bSigma^* = \E \hat\bSigma_{k} = \begin{pmatrix}
	1 & \mathbf{0} \\
	\mathbf{0} & \bSigma
	\end{pmatrix}$ and observe that
	\begin{align*}
		\bDelta_k
		=( \hat\bSigma + \alpha \bI )^{-1/2}(
		\hat{\bSigma}_k - \bSigma^*
		) ( \hat\bSigma + \alpha \bI )^{-1/2}
		-
		( \hat\bSigma + \alpha \bI )^{-1/2}(
		\hat{\bSigma} - \bSigma^*
		) ( \hat\bSigma + \alpha \bI )^{-1/2}.
	\end{align*}
	Let $\bB_k = ( \hat\bSigma + \alpha \bI )^{-1/2}(
	\hat{\bSigma}_k - \bSigma^*
	) ( \hat\bSigma + \alpha \bI )^{-1/2}$. Since
	\begin{equation}\label{eq-lem-dane-quadratic-regularity-1}
		\max_{k\in[m]}\| \bDelta_k \|_2 =
		\max_{k\in[m]}\| \bB_k - \frac{1}{m} \sum_{\ell=1}^{m} \bB_{\ell} \|_2
		\leq 2 \max_{k\in[m]} \| \bB_k \|_2,
	\end{equation}
	it boils down to bound $\| \bB_k\|_2$. To do this, we let $\bA_0 = ( \bSigma^* + \alpha \bI )^{1/2} ( \hat\bSigma + \alpha \bI )^{-1/2}$ and write
	\begin{align*}
		\| \bB_k \|_2 &= \|  \bA_0^{\top} ( \bSigma^* + \alpha \bI )^{-1/2}
		(
		\hat{\bSigma}_k - \bSigma^*
		) ( \bSigma^* + \alpha \bI )^{-1/2} \bA_0 \|_2\\
		&\leq
		\| \bA_0\|_2^2 \| ( \bSigma^* + \alpha \bI )^{-1/2}
		(
		\hat{\bSigma}_k - \bSigma^*
		) ( \bSigma^* + \alpha \bI )^{-1/2} \|_2.
	\end{align*}
	Define $\bD = ( \bI + \alpha (\bSigma^*)^{-1} )^{-1}$ and $\hat{\bSigma}^{(1)}_k = ( \bSigma^* + \alpha \bI )^{-1/2} \hat{\bSigma}_k ( \bSigma^* + \alpha \bI )^{-1/2}$. On the one hand,
	\[
	( \bSigma^* + \alpha \bI )^{-1/2}
	(
	\hat{\bSigma}_k - \bSigma^*
	) ( \bSigma^* + \alpha \bI )^{-1/2}
	= \hat{\bSigma}^{(1)}_k - \bD.
	\]
	On the other hand,
	\begin{align*}
		\| \bA_0 \|_2^2
		&= \| \bA_0 \bA_0^{\top}\|_2
		=\| ( \bSigma^* + \alpha \bI )^{1/2} ( \hat\bSigma + \alpha \bI )^{-1} ( \bSigma^* + \alpha \bI )^{1/2} \|_2\\
		&=\| ( \bSigma^* + \alpha \bI )^{1/2} [ ( \hat\bSigma - \bSigma^* ) + (  \bSigma^* + \alpha \bI ) ]^{-1} ( \bSigma^* + \alpha \bI )^{1/2} \|_2 \\
		&\leq \|  [ ( \bSigma^* + \alpha \bI )^{-1/2}( \hat\bSigma - \bSigma^* )( \bSigma^* + \alpha \bI )^{-1/2} + \bI ]^{-1} \|_2
		=\|  [ \bI + \frac{1}{m} \sum_{k=1}^{m} ( \hat{\bSigma}^{(1)}_k - \bD ) ]^{-1} \|_2 \\
		&\leq \frac{1}{1-\|   \frac{1}{m} \sum_{k=1}^{m} (\hat{\bSigma}^{(1)}_k - \bD)  \|_2}
		\leq  \frac{1}{1- \max_{k\in[m]}\|   \hat{\bSigma}^{(1)}_k - \bD   \|_2},
	\end{align*}
	where we used Lemma \ref{lem-neumann}. Based on these, we have
	\begin{equation}\label{eq-lem-dane-quadratic-regularity-2}
		\max_{k\in[m]} \| \bB_k \|_2
		\leq  \frac{\max_{k\in[m]}\|   \hat{\bSigma}^{(1)}_k - \bD   \|_2}{1- \max_{k\in[m]}\|   \hat{\bSigma}^{(1)}_k - \bD   \|_2},
	\end{equation}
	and it suffices to prove under the given conditions that
	\begin{align}
		\P \left( \max_{k\in[m]} \|   \hat{\bSigma}^{(1)}_k - \bD   \|_2 \leq 1/5 \right) \geq 1- 2me^{-n/C}
		\label{ineq-dane-ls-1}
	\end{align}
	holds for some constant $C$.
	
	By definition, we have
	\begin{align*}
		\hat\bSigma_k^{(1)} - \bD
		&= \begin{pmatrix}
			(1+\alpha)^{-1/2} & \mathbf{0} \\
			\mathbf{0} & ( \bSigma + \alpha \bI )^{-1/2}
		\end{pmatrix}
		\begin{pmatrix}
			\mathbf{0}  & \bar{\bu}_{(k)}^{\top}\\
			\bar{\bu}_{(k)} & \hat\bS_k - \bSigma
		\end{pmatrix}
		\begin{pmatrix}
			(1+\alpha)^{-1/2} & \mathbf{0} \\
			\mathbf{0} & ( \bSigma + \alpha \bI )^{-1/2}
		\end{pmatrix} \notag\\
		&=\begin{pmatrix}
			\mathbf{0}  & (1+\alpha)^{-1/2} [ (\bSigma + \alpha \bI)^{-1/2} \bar{\bu}_{(k)} ]^{\top}\\
			(1+\alpha)^{-1/2} (\bSigma + \alpha \bI)^{-1/2} \bar{\bu}_{(k)} & (\bSigma + \alpha \bI)^{-1/2} ( \hat\bS_k - \bSigma ) (\bSigma + \alpha \bI)^{-1/2}
		\end{pmatrix}
	\end{align*}
	and as a result,
	\begin{align}
		\| \hat\bSigma_k^{(1)} - \bD \|_2\leq
		(1+\alpha)^{-1/2} \| (\bSigma + \alpha \bI)^{-1/2} \bar{\bu}_{(k)} \|_2 + \| (\bSigma + \alpha \bI)^{-1/2} ( \hat\bS_k - \bSigma ) (\bSigma + \alpha \bI)^{-1/2}\|_2.
		\label{ineq-dane-ls-2}
	\end{align}
	Here we used a simple fact that $ \left\| \begin{pmatrix}
	\mathbf{0}  & \bA^{\top} \\
	\bA & \mathbf{0}
	\end{pmatrix} \right\|_2 = \| \bA \|_2$ for any matrix $\bA$.
	
	Observe that $\bv_{i} = (\bSigma + \alpha \bI)^{-1/2} \bu_{i}$ is a sub-gaussian random variable with zero mean and covariance matrix $(1+ \alpha \bSigma^{-1} )^{-1}$.
	On the other hand, $(\bSigma + \alpha \bI)^{-1/2} ( \hat\bS_k - \bSigma) (\bSigma + \alpha \bI)^{-1/2} = \frac{1}{n} \sum_{i\in\mathcal I_k} \bv_{i} \bv_{i}^{\top} - (1+ \alpha \bSigma^{-1} )^{-1} $. Lemma \ref{lem-cov-concentration} forces
	\begin{align*}
		&\P \left(  \|  (\bSigma + \alpha \bI)^{-1/2} \bar{\bu}_{(k)} \|_2 > 1/10 \right) \leq e^{-n/C}, \\
		&\P \left( \| (\bSigma + \alpha \bI)^{-1/2} ( \hat\bS_k - \bSigma ) (\bSigma + \alpha \bI)^{-1/2}\|_2 > 1/10 \right) \leq e^{-n/C},
	\end{align*}
	where $C$ is the constant therein. These estimates and (\ref{ineq-dane-ls-2}) lead to (\ref{ineq-dane-ls-1}).
\end{proof}

Following the similar idea in the proof above, we get the following results.
\begin{lem}\label{lem-dane-quadratic-regularity-N}
	Suppose Assumption \ref{assump-subg-cov-samplesize} holds with $C$ being the constant in Lemma \ref{lem-dane-quadratic-regularity}. Then
	\begin{itemize}
		\item $\mathbb{P} (
		\| (\bSigma^*)^{-1/2}( \hat{\bSigma} - \bSigma^* ) (\bSigma^*)^{-1/2} \|_2
		\leq 1/2 ) \geq 1 - 2me^{-N/C}$;
		\item $\mathbb{P} ( \Delta \leq C_2' \sqrt{p/n} ) \geq 1 - 2 e^{-C_1' p}$ holds for some constants $C_1'$ and $C_2'$.
	\end{itemize}

\end{lem}

Now we come back to the main proof. We will use the three lemmas above to show that with high probability, $\lambda_{\max}(\hat\bSigma) / \lambda_{\min} (\hat\bSigma) \leq 3 \kappa$ and
\begin{align}
	\left\| \hat{\bSigma}^{1/2} \left( \btheta_{t+1} - \hat{\btheta} \right) \right\|_2 \leq
	\left(
	1 - \frac{1- \min\{ 1/2, C_2 p/n \} }{1+ C_3 \alpha }  \right) \left\| \hat{\bSigma}^{1/2} \left(\btheta_t - \hat{\btheta} \right) \right\|_2,\qquad \forall t\geq 0.
	\label{ineq-contraction-pf}
\end{align}
Then we conclude the proof by induction and some simple linear algebra.

First, Lemma \ref{lem-dane-quadratic-regularity-N} asserts that with high probability, $\| (\bSigma^*)^{-1/2}( \hat{\bSigma} - \bSigma^* ) (\bSigma^*)^{-1/2} \|_2
\leq 1/2$. On this event, we have $-\frac{1}{2} \bSigma^* \preceq \hat{\bSigma} - \bSigma^* \preceq \frac{1}{2} \bSigma^*$ and thus $\frac{1}{2} \bSigma^* \preceq \hat{\bSigma} \preceq \frac{3}{2} \bSigma^*$. Hence
$$\frac{1}{2} \lambda_{\min}(\bSigma^*) \leq \lambda_{\min}(\hat\bSigma)\leq \lambda_{\max}(\hat\bSigma)
\leq \frac{3}{2} \lambda_{\max}(\bSigma^*)$$
and $\lambda_{\max}(\hat\bSigma) / \lambda_{\min} (\hat\bSigma) \leq 3\lambda_{\max}(\bSigma^*) / \lambda_{\min} (\bSigma^*)  = 3 \kappa$.

Second, Lemma \ref{lem-dane-quadratic-1} forces
\begin{align}
	\left\| \hat{\bSigma}^{1/2} \left( \btheta_{t+1} - \hat{\btheta} \right) \right\|_2 \leq
	\frac{ 2\Delta^2  + \alpha / \lambda_{\min}(\hat\bSigma) }{ 1 + \alpha / \lambda_{\min}(\hat\bSigma) } \left\| \hat{\bSigma}^{1/2} \left(\btheta_t - \hat{\btheta} \right) \right\|_2,\qquad \forall t\geq 0.
\end{align}
Lemmas \ref{lem-dane-quadratic-regularity} and \ref{lem-dane-quadratic-regularity-N} imply that $\Delta\leq 1/2$, $\Delta \leq \tilde{C} \sqrt{p/n}$, and $\lambda_{\min}(\hat\bSigma) \geq 1$ hold simultaneously with high probability, where $\tilde{C}$ is some constant. On this event, we have
$$
\frac{ 2\Delta^2  + \alpha / \lambda_{\min}(\hat\bSigma) }{ 1 + \alpha / \lambda_{\min}(\hat\bSigma) }=1-\frac{1-2\Delta^2}{1+\alpha/\lambda_{\min}(\hat\bSigma)} \leq 1-\frac{1-\min\{1/2, C_2p/n\}}{1+C_3\alpha}
$$
for some constants $C_2$ and $C_3$. Then we get (\ref{ineq-contraction-pf}) from the estimates above and complete the proof.

\section{Technical lemmas}\label{appendix-lemmas}

The following lemma lists basic properties of strongly convex functions, which can be found in standard textbooks on convex optimization \citep{Nes13}.
\begin{lem}\label{lem-strong-cvxity}
	Suppose $f$ is a convex function defined on some convex open set $\Omega \subseteq \R^p$, and $\partial f(\bx)$ denotes its subdifferential set at $\bx \in \Omega$. The followings are equivalent:
	\begin{itemize}
		\item $f$ is $\rho$-strongly convex in $ \Omega$;
		\item $f[ (1-t) \bx + t \by ] \leq (1-t) f(\bx) + t f(\by) - (\rho/2) t (1-t) \| \by - \bx \|_2^2 $, $\forall \bx, \by \in \Omega$ and $t\in[0,1]$;
		\item $ \langle \bh - \bg , \by - \bx \rangle \geq \rho \| \by - \bx \|_2^2 $, $\forall \bx, \by \in \Omega$, $ \bg \in \partial f(\bx) $ and $ \bh \in \partial f(\by) $.
	\end{itemize}
	If any of the above holds, $f$ is said to be $\rho$-strongly convex.
\end{lem}

\begin{lem}\label{lem-cvx-inverse}
	Let $f:~\R^p \to \R$ be a convex function. Suppose there exists $\bx\in\R^p$ and $r>0$ such that $f$ is $\rho$-strongly convex in $B(\bx,r)$. If $\| \bh - \bg \|_2 < \rho r$ holds for some $\bg \in \partial f(\bx)$ and $\bh \in \partial f(\by)$, then $\| \by - \bx \|_2 \leq \| \bh - \bg \|_2 / \rho \leq r$.
\end{lem}
\begin{proof}
	If we know a priori that $\| \by - \bx \|_2 \leq r$, then we use the strong convexity of $f$ in $B(\bx,r)$ and Cauchy-Schwarz inequality to obtain
	\begin{align*}
		\rho \| \by - \bx \|_2^2 \leq \langle \bh - \bg , {\by} - \bx \rangle \leq \|\bh - \bg \|_2 \| {\by} - \bx\|_2,
	\end{align*}
	and get the desired result. Suppose on the contrary that $\| \by - \bx \|_2 > r$, and define $\bar{\by} = \bx + r (\by - \bx) /\|\by - \bx \|_2$. Then $\| \bar{\by} - \bx \|_2 = r$. The strong convexity of $f$ in $B(\bx,r)$ and Lemma \ref{lem-strong-cvxity} yield
	\begin{align*}
		\langle \bs - \bg , \bar{\by} - \bx \rangle \geq \rho \| \bar{\by} - \bx \|_2^2,\qquad \forall \bs \in \partial f(\bar{\by}).
	\end{align*}
	By the convexity of $f$, we always have
	\begin{align*}
		\langle \bh - \bs , \bar{\by} - \bx \rangle = \frac{r}{ \|\by - \bx \|_2 - r }  \langle \bh - \bs , \by - \bar{\by} \rangle \geq 0,\qquad \forall \bs \in \partial f(\bar{\by}).
	\end{align*}
	Summing up the two inequalities above, we get
	\begin{align*}
		\rho \| \bar{\by} - \bx \|_2^2 \leq \langle \bh - \bg , \bar{\by} - \bx \rangle  \leq \| \bh - \bg\|_2 \| \bar{\by} - \bx \|_2,
	\end{align*}
	where we also used the Cauchy-Schwarz inequality. Then $\| \bh - \bg\|_2 \geq \rho \| \bar{\by} - \bx \|_2 = \rho r$ leads to contradiction.
	Hence, we must have only the case $\| \by - \bx \|_2 \leq r$.
\end{proof}

\begin{lem}\label{lem-prox-nonexpansiveness}
	Let $f:~\R^p \to \R$ be a convex function. For any $\bx,\by\in\R^p$, we have
	\begin{align*}
		\| \mathrm{prox}_{f} (\bx) - \mathrm{prox}_{f} (\by) \|_2^2 \leq \langle \bx-\by, \mathrm{prox}_{f} (\bx) - \mathrm{prox}_{f} (\by) \rangle.
	\end{align*}
	
	If $\inf_{\bx \in \R^p} f(\bx) > -\infty$, then $\| \mathrm{prox}_{f} (\bx) - \bx^* \|_2 \leq \| \bx-\bx^*\|_2$ and
	\begin{align*}
		\| \mathrm{prox}_{f} (\bx) - \bx \|_2^2 \leq \|\bx-\bx^*\|_2^2 - \| \mathrm{prox}_{f} (\bx) - \bx^*\|_2^2
	\end{align*}
	hold for any $\bx^* \in \argmin_{\bx \in \R^p} f(\bx)$.
	
	If $f$ is $\rho$-strongly convex in $B(\bx^*,r)$ for some $r>0$ and $\bx^* = \argmin_{\bx \in \R^p} f(\bx)$, then $\| \mathrm{prox}_{\alpha^{-1} f}(\bx) - \bx^*\|_2 \leq \frac{\alpha}{\alpha+\rho} \| \bx - \bx^*\|_2$, $\forall \bx \in B(\bx^*,r)$ and $\alpha>0$.
\end{lem}

\begin{proof}[\bf Proof of Lemma \ref{lem-prox-nonexpansiveness}]
	The first claim is the well-known ``firm non-expansiveness" property of the proximal mapping \citep{PBo14}.
	
	If $\inf_{\bx \in \R^p} f(\bx) > -\infty$, then any $\bx^* \in \argmin_{\bx \in \R^p} f(\bx)$ is a fixed point of $\mathrm{prox}_f$. The firm non-expansiveness with $\by = \bx^*$ yields
	\begin{align}
		\| \mathrm{prox}_{f} (\bx) - \bx^* \|_2^2 \leq \langle \bx-\bx^*, \mathrm{prox}_{f} (\bx) - \bx^* \rangle
		\label{ineq-proof-lem-prox-nonexpansiveness}
	\end{align}
	and then $\| \mathrm{prox}_{f} (\bx) - \bx^* \|_2 \leq \| \bx-\bx^*\|_2$. The next claim is proved by
	\begin{align*}
		\| \mathrm{prox}_{f} (\bx) - \bx \|_2^2
		&= \| [ \mathrm{prox}_{f} (\bx) - \bx^* ] - (\bx - \bx^*) \|_2^2 \\
		&=\|  \mathrm{prox}_{f} (\bx) - \bx^* \|_2^2 + \| \bx - \bx^* \|_2^2 - 2 \langle \mathrm{prox}_{f} (\bx) - \bx^* , \bx - \bx^* \rangle\\
		&\leq \|  \mathrm{prox}_{f} (\bx) - \bx^* \|_2^2 + \| \bx - \bx^* \|_2^2 - 2 \| \mathrm{prox}_{f} (\bx) - \bx^* \|_2^2\\
		&= \| \bx - \bx^* \|_2^2 - \| \mathrm{prox}_{f} (\bx) - \bx^* \|_2^2,
	\end{align*}
	where the inequality follows from (\ref{ineq-proof-lem-prox-nonexpansiveness}).
	
	For the last claim, we fix any $\alpha>0$ and $ \bx\in B(\bx^*,r)$ and define $\bx^+ = \mathrm{prox}_{\alpha^{-1} f } (\bx)$. Then $\| \bx^+ - \bx^* \|_2 \leq \| \bx - \bx^* \|_2 < r$. The optimality conditions for $\bx^* =\argmin_{\by \in \R^p} f(\by)$ and $\bx^+= \argmin_{\by \in \R^p} \{ f(\by) + (\alpha/2)\|\by-\bx\|_2^2 \}$ imply that $\mathbf{0} \in \partial f(\bx^*)$ and $-\alpha(\bx^+ - \bx) \in \partial f(\bx^+)$. Since $f$ is $\rho$-strongly convex in $B(\bx^*,r)$, Lemma \ref{lem-strong-cvxity} forces
	\begin{align*}
		\rho \|\bx^+ - \bx^* \|_2^2
		&\leq \langle -\alpha(\bx^+ - \bx) - \mathbf{0} , \bx^+ - \bx^* \rangle
		= - \alpha \| \bx^+ - \bx^* \|_2^2 - \alpha \langle \bx^* - \bx  , \bx^+ - \bx^* \rangle \\
		&\leq - \alpha \| \bx^+ - \bx^* \|_2^2 + \alpha \| \bx^* - \bx \|_2 \| \bx^+ - \bx^* \|_2
	\end{align*}
	and thus $\|\bx^+ - \bx^* \|_2 \leq \frac{\alpha}{\alpha+\rho} \| \bx - \bx^* \|_2$.
\end{proof}

\begin{lem}[Neumann expansion]\label{lem-neumann}
	
	Let $\|\cdot\|$be a submultiplicative matrix norm with $\|\bI\|=1$.
	When $\|\bM\|<1$, we have $(\bI-\bM)^{-1}=\sum_{j=0}^{\infty}\bM^{j}=\bI+\bM+\bM(\bI-\bM)^{-1}\bM$,
	$\|\bI-\bM\|\leq1/(1-\|\bM\|)$ and
	$\| (\bI-\bM)^{-1}-(\bI+\bM)\| \leq\|\bM\|^{2}/(1-\|\bM\|)$.
	
\end{lem}

\begin{lem}\label{lem-cov-concentration}
	Let $\bS\succ 0$ and $\alpha\geq 0$ be deterministic, $\bA = ( \bI+ \alpha \bS^{-1} )^{-1}$, $\{ \bu_i \}_{i=1}^n \subseteq \R^{d}$ be i.i.d. sub-gaussian random vectors with zero mean and covariance matrix $\bA$, $\bar{\bu} = \frac{1}{n} \sum_{i=1}^{n} \bu_i$ and $\hat{\bA} = \frac{1}{n} \sum_{i=1}^{n} \bu_i \bu_i^T$. Then the following hold:
	\begin{enumerate}
		\item There exists some positive constant $C$ that only depends on $\| \bu_1 \|_{\psi_2}$, such that $\mathbb{P}(  \|	\bar{\bu} \|_2 > 1/10 ) \leq e^{-n/C}$ and $\P ( \|	\hat{\bA} - \bA \|_2 > 1/10 ) \leq e^{-n/C}$ hold under any one of the two conditions holds: (i) $n\geq C d$; (ii) $\alpha \geq C \Tr(\bS) / n$.
		\item Suppose that $n \geq c d$ for some constant $c>0$. There exist positive constants $C_1'$ and $C_2'$ such that $\P(\max\{\ltwonorm{\bar \bu}, \ltwonorm{\hat\bA-\bA}\}\geq C_2' \sqrt{d/n})\leq 2e^{-C_1' d }$.
	\end{enumerate}
\end{lem}

\begin{proof}[\bf Proof of Lemma \ref{lem-cov-concentration}]
	Let $\lambda_1>\cdots>\lambda_d>0$ be the eigenvalues of $\bS$, and $C>0$ be some constant to be determined. When $n \geq C d$, the fact $\|\bA \|_2 = 1/ (1 + \alpha/\lambda_1) \leq 1$ implies $ n \geq C \Tr(\bA)$. Also, if $\alpha \geq C \Tr(\bS) / n$, then the crude estimate
	\[
	\Tr(\bA) = \sum_{j=1}^{d} \frac{1}{1+ \alpha/ \lambda_j} \leq\sum_{j=1}^{d} \frac{1}{ \alpha/ \lambda_j}= \sum_{j=1}^{d}\frac{\lambda_j}{\alpha}= \frac{\Tr(\bS)}{\alpha} \leq \frac{n}{C_0}
	\]
	also leads to $n \geq C \Tr(\bA)$. Hence it suffices to find some proper $C$ and show the desired results given $n \geq C \Tr(\bA) $.
	
	Now we prove the first statement. We first study concentration of the sample mean vector $\bar{\bu}$. Since $\bar{\bu}$ is a sub-gaussian random vector with covariance matrix $n^{-1} \bA$, Theorem 2.1 in \cite{HKZ12} asserts the existence of a constant $c_1>0$ such that
	\begin{equation}\label{eq-lem-cov-concentration-1}
		\P \left[ \| \bar\bu \|_2^2
		\leq c_1 n^{-1} \left(  \Tr ( \bA ) + 2 \sqrt{ \Tr (\bA^2) t } + 2 \| \bA \|_2 t \right) \right] \geq 1-e^{-t},\qquad \forall t>0.
	\end{equation}
	Choose any constant $C_1 \geq 500 c_1$. Let $t=n/C_1$, and suppose that $n\geq C_1 \Tr(\bA)$. Using $\|\bA\|_2 \leq 1$ and $\Tr (\bA^2) \leq \Tr (\bA) \| \bA\|_2 \leq \Tr ( \bA  )$, we get
	\begin{align*}
		&c_1 n^{-1} \left( \Tr ( \bA ) + 2 \sqrt{ \Tr (\bA^2) t } + 2 \| \bA \|_2 t  \right) \leq \frac{c_1}{C_1} + \frac{2c_1 \sqrt{ \Tr(\bA) t }}{n} + \frac{2 c_1 t}{n}\\
		&= \frac{c_1}{C_1} + \frac{2c_1 \sqrt{ \Tr(\bA) n/C_1 }}{n} + \frac{2 c_1 (n/C_1 ) }{n}
		= \frac{3c_1}{C_1} + 2c_1 \sqrt{ \frac{ \Tr(\bA)}{ C_1 n } } \leq \frac{5 c_1}{C_1} \leq \frac{1}{10^2}.
	\end{align*}
	Hence $\P \left( \| \bar\bu \|_2 > 1/10 \right) \geq e^{-n/C_1}$.
	
	Now we come to concentration of the sample covariance matrix $\hat \bA$.
	Let $r(\bA) = \Tr(\bA) / \| \bA \|_2 $.
	According to Theorem 9 in \cite{KLo14}, there exists a constant $c_2 \geq 1$ such that the following holds: for any $t \geq 1$, with probability at least $1-e^{-t}$ we have
	\begin{equation}\label{eq-lem-cov-concentration-2}
		\| \hat{\bA}-\bA \|_2
		\leq c_2 \ltwonorm{\bA}\max\left\{\sqrt{\frac{r(\bA)}{n}},\frac{ r(\bA)}{n},\sqrt{\frac{t}{n}}, \frac{t}{n} \right\}.
	\end{equation}
	Note that the upper bound above can be rewritten as
	\begin{equation}\label{eq-lem-cov-concentration-3}
		c_2 \max\left\{\sqrt{\frac{\ltwonorm{\bA} \Tr(\bA)}{n}},\frac{ \Tr(\bA)}{n},\ltwonorm{\bA}\sqrt{\frac{t}{n}}, \ltwonorm{\bA}\frac{t}{n} \right\}
		\leq  c_2 \max\left\{\sqrt{\frac{\Tr(\bA)}{n}},\frac{ \Tr(\bA)}{n}, \sqrt{\frac{t}{n}}, \frac{t}{n} \right\}.
	\end{equation}
	Let $C_2 = 100 c_2^2$. When $n \geq C_2 \Tr(\bA) $, by taking $t=n/C_2$ we get $\mathbb{P}( \|	\hat{\bA} - \bA \|_2> 1/10 ) \leq e^{-n/C_2}$.	
	The proof of the first statement is then finished by taking $C=\max\{ C_1,C_2 \}$.
	
	We proceed to prove the second statement. Let $t=C_1' d$ for some constant $C_1'$. Note that $d \geq \Tr(\bA)$ and $t \geq C_1' \Tr(\bA)$. According to (\ref{eq-lem-cov-concentration-1}), (\ref{eq-lem-cov-concentration-2}) and (\ref{eq-lem-cov-concentration-3}), we obtain that with probability at least $1-2e^{-C_1' d}$,
	$
	\ltwonorm{\bar\bu}\leq \tilde{C} \sqrt{ d/n}
	$
	and
	$
	\| \hat\bA - \bA \|_2 \leq \tilde{C} \max\left\{\sqrt{ d/n}, d/n\right\}
	$
	hold with some constant $\tilde{C}$. Since $n\geq c d$, we have $$ \max\left\{\sqrt{ d/n}, d/n\right\} \leq
	\max\{ 1,1/\sqrt{c} \} \sqrt{ d/n }  .$$ By combining the inequalities above, we obtain that
	$$
	\P(\max\{\ltwonorm{\bar \bu}, \ltwonorm{\hat\bA-\bA}\}\geq C_2' \sqrt{d/n})\leq 2e^{-C_1' d}
	$$
	with $C_2' = \tilde C \max\{1,1/\sqrt{c}\}$.
\end{proof}

\bibliographystyle{ims}
\bibliography{bib}

\end{document}